\documentclass{IEEEtran}
\usepackage{latexsym,amssymb,amsmath,graphicx,epsf,cite,bbm,float}
\usepackage{ifpdf,flexisym}
\usepackage{epstopdf}
\usepackage{caption}
\usepackage{subcaption}
\usepackage{amsfonts,mathrsfs}
\usepackage{empheq}
\usepackage{bbm}
\usepackage{mathtools}
\usepackage{subfloat}
\usepackage{amsthm}
\usepackage{lipsum}
\usepackage{stfloats}
\usepackage{enumitem,nicefrac,dsfont,bm}
\usepackage{tikz}
\usetikzlibrary{arrows,fit}
\usepackage{mleftright}
\usepackage{xparse}
\NewDocumentCommand{\evalat}{sO{\big}mm}{%
	\IfBooleanTF{#1}
	{\mleft. #3 \mright|_{#4}}
	{#3#2|_{#4}}%
}
\newlength{\arrow}
\DeclarePairedDelimiter\ceil{\lceil}{\rceil}
\DeclarePairedDelimiter\floor{\lfloor}{\rfloor}
\makeatletter
\newcommand{\@giventhatstar}[2]{#1 \;\middle|\; #2}
\newcommand{\@giventhatnostar}[3][]{\;#1 #2\;#1|\;#3 #1}
\newcommand{\giventhat}{\@ifstar\@giventhatstar\@giventhatnostar}
\makeatother

\newcommand{\nn}{\nonumber}
\newcommand{\yay}[1]{\left(#1\right)}
\newcommand{\brac}[1]{\left[#1\right]}
\newcommand{\lb}{\left\{}
\newcommand{\rb}{\right\}}
\newcommand{\set}[1]{\lb#1\rb}
\newcommand{\abs}[1]{\left|#1\right|}

\newcommand{\norm}[1]{\left\| #1\right\|}
\newcommand{\pmf}[2]{p^{}_{#1}\left(#2\right)}

\newcommand{\prob}[1]{\mathbb{P}\left(#1\right)}
\newcommand{\condprob}[2]{\mathbb{P}\left(\giventhat*{#1}{#2}\right)}
\newcommand{\expt}[1]{\mathbb{E}\left[#1\right]}

\newcommand{\pp}[1]{\tilde{p}_{#1}}
\newcommand{\pq}[1]{\tilde{q}_{#1}}
\newcommand{\vecnot}[2]{\mathbf{#1}_{\setminus#2}}
\newcommand{\pest}[2]{\hat{\rho}_{#1}\left(#2\right)}
\newcommand{\pestvec}[1]{\hat{\bm{\rho}}\left(#1\right)}
\newcommand{\dvec}[1]{\bm{#1}}
\newcommand{\condexpt}[2]{\mathbb{E}\left[\giventhat*{#1}{#2}\right]}
\newcommand{\ind}[1]{\mathds{1}\left(#1\right)}
\newcommand{\sgn}[1]{\sign\left(\cdot\right)}
\newcommand{\rans}[1]{\mathcal{R}\left(\cdot\right)}
\newcommand{\nuls}[1]{\mathcal{N}\left(\cdot\right)}
\newcommand{\domain}[1]{\mathcal{D}\left(\cdot\right)}

\newcommand{\pn}[1]{p_{\setminus#1}}
\newcommand{\qn}[1]{q_{\setminus#1}}

\newcommand{\epsn}[1]{\varepsilon_{\setminus#1}}

\newcommand{\ppvec}[1]{\tilde{\bm{p}}_{#1}}
\newcommand{\V}[1]{V\left(#1\right)}

\newcommand{\xvec}[1]{\mathbf{x}_{#1}}

\newcommand{\Xvec}[1]{\mathbf{X}_{#1}}

\newcommand{\rvec}[1]{\mathbf{#1}}
\newcommand{\Xnot}[1]{\mathbf{X}_{\setminus{#1}}}
\newcommand{\xnot}[1]{\mathbf{x}_{\setminus{#1}}}

\newtheorem{lemma}{Lemma}

\newtheorem{corollary}{Corollary}
\newtheorem{thm}{Theorem}
\newtheorem{prop}{Proposition}

\newtheorem{definition}{Definition}
\DeclareMathOperator{\var}{Var}
\DeclareMathOperator{\sign}{sign}

\DeclareMathOperator*{\argmin}{arg\,min}
\DeclareMathOperator*{\argmax}{arg\,max}
\tikzstyle{circ}=[draw, circle, minimum size=2em]
\tikzstyle{box}=[draw, rectangle, minimum height=2em, minimum width=4em]
\tikzstyle{bigbox}=[draw, rectangle, minimum height=3em, minimum width=2em]
\tikzstyle{init} = [pin edge={to-,thin,black}]
\IEEEoverridecommandlockouts
\title{Unsupervised Opinion Aggregation -- A Statistical Perspective}
\author{Noyan~C.~Sev\"uktekin~and~Andrew~C.~Singer\\Department of Electrical and Computer Engineering\\
University of Illinois at Urbana-Champaign\\
acsinger@illinois.edu}%
\begin{document}
	\maketitle
	\begin{abstract}
		Complex decision-making systems rarely have direct access to the current state of the world and they instead rely on opinions to form an understanding of what the ground truth could be. Even in problems where experts provide opinions without any intention to manipulate the decision maker, it is challenging to decide which expert's opinion is more reliable -- a challenge that is further amplified when decision-maker has limited, delayed, or no access to the ground truth after the fact. This paper explores a statistical approach to infer the competence of each expert based on their opinions without any need for the ground truth. Echoing the logic behind what is commonly referred to as \textit{the wisdom of crowds}, we propose measuring the competence of each expert by their likeliness to agree with their peers. We further show that the more reliable an expert is the more likely it is that they agree with their peers. We leverage this fact to propose a completely unsupervised version of the na\"{i}ve Bayes classifier and show that the proposed technique is asymptotically optimal for a large class of problems. In addition to aggregating a large block of opinions, we further apply our technique for online opinion aggregation and for decision-making based on a limited the number of opinions.  
	\end{abstract}
	\begin{IEEEkeywords}
		Unsupervised opinion aggregation, statistical inference, stochastic experts.
	\end{IEEEkeywords}
	\section{Introduction}\label{SEC:Intro}
	\IEEEPARstart{M}{any} modern engineering systems have access to extrinsic information in the form of \textit{opinions} -- subjective evaluations of the current state of the world. Contrary to the conventional models of information flow capturing the uncertainty in tasks with quantifiable objectives such as measuring a physical quantity or transmitting information through a medium, opinions are generated for tasks with inherent subjectivity, thus providing limited insight on the underlying rationale or the proficiency of the opinion source.  
	
	Albeit a more unforgiving form of information for decision making, opinions are critical in high-risk applications including autonomous driving \cite{geiger2012we}, personalized medicine \cite{hamburg2010path}, and disaster detection \cite{bahrepour2010distributed}, where information is readily available from highly heterogeneous opinion sources, or \textit{experts}, while the true state of the world remains hidden, or unachievable. \textit{The internet of things} is conjectured to be particularly rich in subjective information as highly localized processing units under power, circuit area, and latency constraints are designed to collectively address global computational tasks \cite{zhang2018approaches}. Complicated physico-chemical processes, such as genomic data sequencing \cite{ochoa2016genomic}, are often subject to high-dimensional uncertainty, while the target efficiency is remarkably high. Similarly, human integration into labeling data exhibits a similar, difficult-to-model, behavior \cite{snow2008cheap}.
	
	Decision-making by consulting experts, or opinion aggregation, has far-reaching roots: Bayesian hypothesis testing sets the fundamental limits of opinion aggregation provided that the complete probability law governing the experts exists and is known to the decision maker \cite{poor2013introduction}. Nonetheless, it is difficult, if at all possible, to model reliably, over the entire application space, the probability law that underlies highly-specialized processing units, which are often trained on limited data. The modeling error and the concomitant performance degradation in opinion aggregation might prove to be uncontrollable in such cases. The uncertainty in modeling, however, does not necessarily render statistical inference of the underlying probability law implausible. Dempster-Schafer theory addresses the Bayesian inference problem under model uncertainty by incorporating \textit{belief} and \textit{plausibility} functions to substitute direct application of the probability law \cite{dempster1968generalization,shafer1976mathematical}. In general, use of approximate models for reasoning is called \textit{fuzzy logic} \cite{zadeh1988fuzzy} and it has a wide range of applications in control theory \cite{lee1990fuzzy}.
	
	In the absence of approximate or complete probability laws, \textit{feedback} is often used to learn and mitigate the reliability of each source and the concomitant probability distribution over opinions. The mixture of experts setup successfully aggregates opinions from potentially adversarial experts by the use of reliable past information \cite{cesa2006prediction}. Similarly, boosting and the associated meta-learning concepts use feedback, often at the expense of additional training data, to learn and use the reliability of different strategies, or classifiers \cite{schapire2012boosting}. The use of feedback is both the strength and the fundamental limitation of such strategies, as feedback is often expensive even when it is feasible to generate. The Bayesian ideas are called to action and a prior distribution on the underlying probability law is often \textit{assumed} when feedback is not feasible. Bayesian approach enables the rather powerful toolbox of iterative algorithms including belief propagation, expectation-maximization, and mean-field methods \cite{ok2016optimality,welinder2010online,zhang2014spectral,dawid1979maximum,baharad2011distilling}. Bayesian decision-aggregation methods are powerful to the extent with which the prior successfully captures reality.  
	
	The existence of reliable or approximate models, feedback, or prior information fundamentally defines the application space on which the associated ideas can be employed, and thus, are limited to such applications. As such, we may consider opinion-aggregation techniques that use some form of side information as \textit{supervised} rules. Conceptually, the presence of such side information renders an otherwise vital fact obsolete: Experts aim to achieve a common task, not to fail it, as long as a fixed probability law exists. The opinion-aggregation rules that rely solely on the existence of a probability law can therefore be considered \textit{unsupervised} rules. Often a community of experts is less subjective than its individual constituents, motivating statistical inference to compensate the lack of supervision. The fundamental challenge is to design unsupervised decision rules that reliably aggregate opinions without using any side information on the hidden-yet-fixed probability law governing the experts. 
	
	This paper addresses this challenge from an inherently statistical perspective. Section \ref{SEC:ProbDef} provides the formal background for opinion aggregation and identify different regimes of operation. Section \ref{SEC:PseudoComp} introduces a novel technique for estimating, dynamically in real-time, the reliability of each expert from a set of opinions and discusses the properties of the proposed method. The purpose of such reliability estimation is made clear by using these estimates to infer the unknown probability law. Section \ref{SEC:Naive} introduces the a sharp upper-bound on the minimum probability of error achievable when the probability law is known, improving upon the state-of-the-art. Section \ref{SEC:Pseudo_naive} proposes an unsupervised opinion aggregation rule based on the minimum probability of error rule that uses the unsupervised reliability estimates and investigate its fundamental limits. Section \ref{SEC:EmpiricalPseudoNaiveBayes} addresses empirical extensions that aggregate a fixed block of opinions as well as doing so adaptively in real-time. Experiments are given in  Section \ref{SEC:Exp}. Proofs are deferred to the appendix.
	\section{Notation, Background, and Problem Definition}
	\label{SEC:ProbDef}
	We first discuss the notation used in this paper, then formally define the concepts of opinions, experts, tasks, and opinion aggregation. We address different modes of unsupervised opinion aggregation separately.  
	%
	%
	%
	%%%%%%%%%%%%%%%%%%%%%%%%%%%%%%%%%%%%%%%%%%%%%%%%%%%%%%%%%%%%%%%%%%%%%%%%%%%%%%%%%%%%%%%%%%%%%%%%%%%%%%%%%%%%%%%%%%%%%%%%%%%%%%%%%%%
	%
	%
	\subsection{Notation}
	A probability space is a triplet $\yay{\Omega, \mathscr{F}, \mathbb{P}}$, where $\Omega$ is the event space, $\mathscr{F}$ is the sigma-field defined on $\Omega$, and $\mathbb{P}$ is the probability measure. Random variables are denoted by capital letters and the corresponding samples are denoted by lowercase letters: $(X,x)$. Independent random variables $\yay{X_1, X_2}$ are denoted by $X_1 \perp X_2$ and conditionally independent random variables $\yay{X_1, X_2}$ conditioned on $Y$ are denoted by $X_1 - Y - X_2$. A random process is an indexed collection of random variables $\set{X(t) : t\in\mathbb{T}}$, where $\mathbb{T}$ is the index set with cardinality $\abs{\mathbb{T}}=T$. Expectation, conditional expectation, and conditional probability operators are $\expt{\cdot}$, $\condexpt{\cdot}{\cdot}$, and $\condprob{\cdot}{\cdot}$ respectively. Random vectors and corresponding samples are column vectors, denoted by bold letters: $\dvec{X} = \brac{X_1,\cdots, X_N}^{\top}$, $\dvec{x} = \brac{x_1,\cdots, x_N}^{\top}$. A vector that excludes a specific element $\yay{X_j,x_j}$ is denoted by $\yay{\Xnot{j},\xnot{j}}$. When a random process takes vector values, we use the notation $\yay{\dvec{X}\yay{t}, \dvec{x}\yay{t}}$, $t\in\mathbb{T}$. The corresponding random matrices are constructed explicitly: $\brac{\dvec{X}(1),\cdots,\dvec{X}(T)}$. The indicator function is denoted by $\ind{\cdot}$, where the domain is to be understood from context. We use  $[N] \triangleq \set{1,\cdots,N}$ to denote the natural numbers up to a finite limit. Standard Hilbert space notation applies, specifically $\abs{\cdot}$ denotes absolute value and $\norm{\cdot}_{p}$ denotes $p$-norm.
	%
	%
	%
	%%%%%%%%%%%%%%%%%%%%%%%%%%%%%%%%%%%%%%%%%%%%%%%%%%%%%%%%%%%%%%%%%%%%%%%%%%%%%%%%%%%%%%%%%%%%%%%%%%%%%%%%%%%%%%%%%%%%%%%%%%%%%%%%%%
	%
	%
	%
	\subsection{Background}
	
	Let a set of \textit{tasks} $\mathbb{T}$ for identifying hidden binary \textit{states} $Y(t)\in\set{-1,1}$, $\forall t\in\mathbb{T}$ (in the context of classification, $Y(t)$ is often called \textit{label or ground-truth} instead) be generated independently, $Y(t)\perp Y(\tau)$, $\forall t\neq \tau$, with a uniform prior:
	\begin{equation}
		\label{uniform_prior}
		\prob{Y(t)=1}=\prob{Y(t)=-1} = \nicefrac{1}{2}, \forall t\in\mathbb{T}.
	\end{equation}
	Even though there exists applications that do not admit the uniform prior, such as group testing \cite{dorfman1943}, where populations are severely skewed, \eqref{uniform_prior} is a common assumption in opinion aggregation problems \cite{berend2015finite}. If a non-uniform prior on $Y$ exists, standard Bayesian decision-making principles can be used to adapt the results presented in this paper \cite{poor2013introduction}.  
	
	Experts generate binary opinions $X_i  \in \set{-1,1}$, $\forall i\in[N]$, that identify the true state $Y(t)$ with some probability:
	\begin{equation}
		\label{true_competence}
		p_i \triangleq \prob{X_i(t)=Y(t)}, \forall t\in\mathbb{T}.
	\end{equation} 
	The \textit{true competence} $p_i$ of an expert is considered \textit{fixed} across tasks. Let there be $N$ experts and assume that opinions are \textit{generated} independently, which amounts to:  
	\begin{equation}
		\label{opinion_gen}
		X_i(t)-Y(t)-X_j(t) , \forall i\neq j \in [N],
	\end{equation}
	for every task $t\in\mathbb{T}$. One should note that \eqref{opinion_gen} does \textit{not} indicate that opinions are independent. Indeed, it is conceptually clear that opinions should be dependent for meaningful inference as an opinion is a subjective evaluation of the current state, not an arbitrary input. 
	
	Experts that can reliably be defined by a stochastic law, as done here via $\eqref{uniform_prior}-\eqref{opinion_gen}$, are sometimes called \textit{stochastic experts} to separate them from the more game-theoretic framework of \textit{adversarial experts} \cite{cesa2006prediction}, a similar distinction exists for multi-armed bandits \cite{bubeck2012regret}. Furthermore, when $p_i$ is a function of the underlying task space $\mathbb{T}$, different from how it is defined here, experts exhibit task-dependent competences. Neyman-Pearson formulation of binary hypothesis testing \cite{neyman1933ix} and the two-coin Dawid-Skene model \cite{dawid1979maximum}, are examples of expert competences changing over the task space $\mathbb{T}$. 
	
	Opinion-based systems are sometimes referred to as \textit{semi-supervised} systems due to the availability of experts and the concomitant subjective information \cite{murphy2012machine}. However, in the context of opinion aggregation, one may consider supervision as \textit{any} form of side-information that yields inference of the underlying probability law beyond the extent that  opinions alone would allow. Formally, a \textit{supervised} opinion-aggregation rule refers to a function $f\yay{\cdot}$ mapping a set of opinions $\dvec{X}\yay{t}$ to an estimate of the true state $Y(t)$:
	\begin{equation}
		\label{supervised}
		\hat{Y}\yay{t} = f\yay{\dvec{X}\yay{t}; \mathcal{S}},
	\end{equation}
	where $\hat{Y}(t) \in \set{-1,1}$ and $\mathcal{S}$ denotes some form of side information. Supervised opinion-aggregation methods require different forms of side-information $\mathcal{S}$: $\set{p_1,\cdots,p_N}$ for Binary hypothesis testing \cite{poor2013introduction}, subsets of $\set{y_1,\cdots,y_{t-1}}$ for boosting and mixture of experts \cite{cesa2006prediction, schapire2012boosting}, or a priori distribution $\pmf{P_1,\cdots,P_N}{\cdot}$ on competences for Bayesian techniques; \cite{ok2016optimality,welinder2010online,zhang2014spectral,dawid1979maximum,baharad2011distilling} are some of the prominent examples. 
	
	\textit{Unsupervised} opinion aggregation, on the other hand, refers to functions that only rely on the existence of the underlying probability law, here, for instance, characterized by \eqref{uniform_prior}-\eqref{opinion_gen}. As such, they assume the form:
	\begin{equation}
		\label{unsupervised}
		\hat{Y}\yay{t} = f\yay{\dvec{X}\yay{t}}.
	\end{equation}
	The most prominent example of unsupervised decision aggregation is \textit{majority voting}, which is commonly accepted as the \textit{baseline} for unsupervised techniques. A modern unsupervised technique is called spectral meta learner and it relies on the singular values of the empirical covariance matrix of a collection of opinions \cite{parisi2014ranking}. 
	
	Binary opinion-aggregation techniques often aim to minimize the probability of error:
	\begin{equation}
		\label{proberr}
		\min_{f\in\mathcal{D}} \prob{f\yay{\dvec{X}\yay{t}} \neq Y\yay{t}},
	\end{equation}
	where $\mathcal{D}$ is the family of admissible opinion-aggregation rules that depend on the problem setup. For instance, the admissible rules comprise all, potentially randomized, decision rules when competences are known, past-measurable decision rules when feedback from past tasks is available, and functions that directly map available opinions to decisions on the corresponding tasks in the absence of all side information. 
	
	Opinion-aggregation rules \eqref{supervised}-\eqref{unsupervised}, and the corresponding performance metric \eqref{proberr}, are written in the form of single-task opinion aggregation, where a set of opinions $\set{X_1(t),\cdots,X_N(t)}$ are used to make a decision $\hat{Y}(t)$. Next, we discuss different modes of operation for unsupervised opinion aggregation.   
	
	%Generally, a decision aggregation rule might map opinions on any subset of tasks from any subset of experts, $\set{X_i\yay{t} : \yay{i,t} \in \mathcal{N} \times \mathcal{T}}$ where $\mathcal{N} \times \mathcal{T} \subset \mathbb{T}\times [N]$, to the corresponding state estimates $\set{\hat{Y}\yay{t} : t\in\mathcal{T}}$, leading to different operational meanings such as instantaneous, block-processing, or adaptive decision aggregation, as addressed next.
	
	% 
	%
	%
	%%%%%%%%%%%%%%%%%%%%%%%%%%%%%%%%%%%%%%%%%%%%%%%%%%%%%%%%%%%%%%%%%%%%%%%%%%%%%%%%%%%%%%%%%%%%%%%%%%%%%%%%%%%%%%%%%%%%%%%%%%%%%%%%%%%
	%
	%
	\subsection{Unsupervised Opinion Aggregation}
	
	%%
	%\begin{figure*}[t!]
	%	\centering
	%	\begin{subfigure}[t]{0.175\textwidth}
		%		\centering
		%		\includegraphics[width=\textwidth]{static.jpg}
		%		\caption{Instantaneous Decision Aggregation}
		%	\end{subfigure}%
	%	~ 
	%	\begin{subfigure}[t]{0.375\textwidth}
		%		\centering
		%		\includegraphics[width=\textwidth]{block.jpg}
		%		\caption{Block(-iterative) Decision Aggregation}
		%	\end{subfigure}
	%	~ 
	%	\begin{subfigure}[t]{0.375\textwidth}
		%		\centering
		%		\includegraphics[width=\textwidth]{adaptive.jpg}
		%		\caption{Adaptive Decision Aggregation}
		%	\end{subfigure}
	%	\caption{An Overview of Unsupervised Decision Aggregation Rules}
	%	\label{fig:unsupervised_overall}
	%\end{figure*}
	
	\textit{Unsupervised opinion aggregation} refers to employing a function $f(\cdot)$ of opinions $\set{X_i\yay{t} : \yay{i,t} \in \mathcal{N} \times \mathcal{T}}$ for some subset $\mathcal{N} \times \mathcal{T} \subset \mathbb{T}\times [N]$ to identify a set of hidden states $\set{Y(t): t\in \mathcal{T}}$, \textit{directly}:
	\begin{equation}
		\label{too_general}
		\set{Y(t): t\in \mathcal{T}} = f\yay{\set{X_i\yay{t} : \yay{i,t} \in \mathcal{N} \times \mathcal{T}}}
	\end{equation}
	Note that there is no side information $\mathcal{S}$ as an input to the decision rule and that the subset $\mathcal{N}\times \mathcal{T}$ determines the \textit{operational meaning} of the opinion-aggregation rule: The function $f\yay{\cdot}$ might be fixed for all tasks $\mathbb{T}$, such as majority voting \cite{marquis1785essai}, or it might change adaptively in tasks. Furthermore, $f\yay{\cdot}$ might process blocks of opinions (often iteratively and non-adaptively) \cite{ok2016optimality,welinder2010online,baharad2011distilling,dawid1979maximum,zhang2014spectral}. The formal definitions of unsupervised opinion-aggregation rules are discussed below.
	
	\subsubsection{Instantaneous Opinion Aggregation}
	A fixed function directly applicable to opinions $\set{X_i(t): i\in [N]}$ on any task $t\in\mathbb{T}$, is an instantaneous opinion-aggregation strategy. Conceptually, these rules do not require additional memory to store past opinions. Formally, they take form $\hat{Y}(t) = f\yay{\dvec{X}\yay{t}}$, as illustrated in figure \ref{fig:insta}.
	\begin{figure}[h!]
		\begin{tikzpicture}[node distance=1.5cm,auto,>=latex',thick]
			\node (init) [] {$Y(t)$};
			\node [box, right of=init, node distance = 1.75cm, text width= 1.5cm, align=center] (a) {Consulting Experts};
			\node (c) [right of=a, node distance=2.25cm] {$\begin{bmatrix}
					X_1(t) \\ \vdots \\ X_N\yay{t}
				\end{bmatrix}$};
			\node [box, right of=c, node distance=2.25cm, text width= 1.75cm, align=center] (d) {Aggregating Opinions $f(\cdot)$};
			\node (end) [right of=d, node distance= 1.75cm]{$\hat{Y}(t)$};
			\path[->] (init) edge node {} (a);
			\path[->] (a) edge node {} (c);
			\draw[->] (c) edge node {} (d) ;
			\draw[->] (d) edge node {} (end) ;
		\end{tikzpicture}
		\caption{Instantaneous Opinion Aggregation} \label{fig:insta}
	\end{figure}
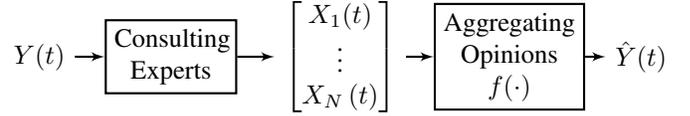

	It is often difficult to find meaningful unsupervised rules that are instantaneous, with a notable exception: Majority voting, denoted by $\V{\dvec{X}(t)}$, is a commonly-accepted baseline for unsupervised rules:
	\begin{equation}
		\label{def:MV}
		\V{\dvec{X}(t)} = \sign\yay{\sum_{i=1}^{N}X_i(t)},
	\end{equation}
	where ties are broken arbitrarily. It is often taken for granted that ties being broken arbitrarily is a direct consequence of Assumption \ref{uniform_prior}. Indeed, if a prior $\pmf{Y}{y}$ on $Y(t)$ were known, 
	\begin{equation*}
		\hat{Y} = \argmax_{y\in\set{-1,1}}\prob{Y=y}
	\end{equation*}
	should be chosen in the event of a tie \cite{poor2013introduction}. 
	
	Conjectures of Marquis de Condorc\'{e}t, \cite{marquis1785essai}, has long been debated in the social choice literature i.a. \cite{paroush1997stay,berend1998condorcet,ben2000nonasymptotic}, revealing that majority voting is not reliable when heterogeneous ($p_i\in[0,1]$) or arbitrarily weak ($p_i\rightarrow \nicefrac{1}{2}$) populations of experts are concerned.

	\subsubsection{Block(-Iterative) Opinion Aggregation}
	A strategy that processes a \textit{collection} of opinions $\set{X_i\yay{t} : \yay{i,t} \in \mathcal{N} \times \mathcal{T}}$ to estimate the corresponding states $\set{\hat{Y}(t): t\in\mathcal{T}}$ is a block opinion-aggregation rule. Conceptually, these rules, often iteratively, process past opinions to decide for the respective block of tasks. We focus on block rules that leverage \textit{all} the available opinions with $\mathcal{N}=[N]$ and $\mathcal{T} = \mathbb{T}$, as illustrated in figure \ref{fig:block}.
	
	\begin{figure}[h!]
		\begin{tikzpicture}[auto,>=latex', thick]
			\node (init) [] {$\begin{matrix}
					Y(1) & \cdots & Y(T)
				\end{matrix}$};
			\node [box, below of=init, node distance = 1.5cm, text width= 1.5cm, align=center] (a) {Consulting Experts};
			\node [right of=a, node distance=2.75cm] (ghost) {};
			\node (data) [below of=ghost, node distance=1cm] {$\begin{bmatrix}
					X_1(1) & \cdots & X_1(T)\\ \vdots& \ddots & \vdots \\ X_N(1) & \cdots & X_N(T)
				\end{bmatrix}$};
			\node [box, right of=ghost, node distance=3cm, text width= 2cm, align=center] (d) {Aggregating Opinions \\ $f(\cdot)$};
			\node (end) [above of=d, node distance= 1.5cm]{$\begin{matrix}
					\hat{Y}(1) & \cdots & \hat{Y}(T)
				\end{matrix}$};
			\node[draw, dashed, line width=0.7pt, fit=(a)(d)(data)](boxed){};  
			\draw[->] (init) edge node {} (a);
			\draw[->] (a.south) |- (data.west);
			\draw[->] (data.east) -|  (d.south) ;
			\draw[->] (d) edge node {} (end) ;
			\draw[dotted, ->] (init) -- node[below] {\textit{Block Inference}} (end);
		\end{tikzpicture}
		\caption{Block(-iterative) Opinion Aggregation} \label{fig:block}
	\end{figure}
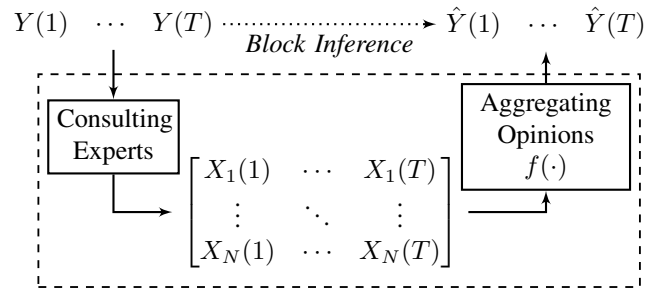
	Often, off-line techniques such as expectation maximization, belief propagation, or spectral decomposition methods are used iteratively on a collection of opinions \cite{ok2016optimality,welinder2010online,baharad2011distilling,dawid1979maximum,zhang2014spectral}. Specifically, belief propagation has been shown to asymptotically minimize the probability of error for specific sparse subsets $\mathcal{N}\subset [N]$, often referred to as task-assignment \cite{ok2016optimality}.  The singular value distribution of the empirical covariance matrix of opinions has also been investigated for opinion aggregation \cite{zhang2014spectral,parisi2014ranking}. Adaptations of expectation maximization algorithm have been proposed for adaptive block-processing and task-dependent modeling of competences   \cite{welinder2010online,dawid1979maximum}. These methods are generally computationally expensive and they seldom yield provable guarantees for their performance.   
	
	\subsubsection{Adaptive Opinion Aggregation}
	A strategy that infers the underlying probability law sequentially from past observations is an adaptive opinion-aggregation strategy. Formally, they have the form $\hat{Y}(t) = f_t\left(\mathbf{X}(t)\right)$, as illustrated in figure \ref{fig:adaptive}.
	\begin{figure}[h!]
		\begin{tikzpicture}[node distance=1.5cm,auto,>=latex',thick]
			\node (init) [anchor=west] {$\hspace{-2em}Y(t)$};
			\node [box, right of=init, node distance = 1.5cm, text width= 1.5cm, align=center] (a) {Consulting Experts};
			\node (c) [right of=a, node distance=2.25cm] {$\begin{bmatrix}
					X_1(t) \\ \vdots \\ X_N\yay{t}
				\end{bmatrix}$};
			\node [box, right of=c, node distance=2.25cm, text width= 1.75cm, align=center] (d) {Aggregating Opinions $f_t(\cdot)$};
			\node (end) [right of=d, node distance= 1.75cm]{$\hat{Y}(t)$};
			\node (past) [below of= init, node distance = 2cm] {$\begin{matrix}
					Y(t-1) \\ \vdots \\ Y(1)
				\end{matrix}$};
			\node (memory) [right of=past, node distance=3.5cm] {$\begin{bmatrix}
					X_1(1) & \cdots & X_1(t-1)\\ \vdots& \ddots & \vdots \\ X_N(1) & \cdots & X_N(t-1)
				\end{bmatrix}$}; 
			\node[draw, dashed, line width=0.7pt, fit=(past)(memory)](boxedpast){};  
			\node [right of= boxedpast, node distance= 4.5cm, text width= 1.5cm, align = center] () {\textit{Memory or Inference}};
			\path[->] (init) edge node {} (a);
			\path[->] (a) edge node {} (c);
			\draw[->] (c) edge node {} (d);
			\draw[->] (d) edge node {} (end);
			\draw[->] (memory.east) -|  (d.south);
			\draw[->] (past) edge node {} (memory) ;
		\end{tikzpicture}
		\caption{Adaptive Opinion Aggregation} \label{fig:adaptive}
	\end{figure}
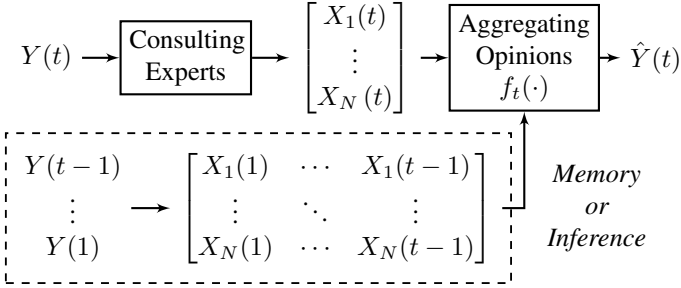
	
	These strategies often estimate and employ empirical competence estimates $\set{\hat{p}_1,\cdots,\hat{p}_N}$ in the decision making process and, traditionally, they are almost exclusively formulated as supervised opinion-aggregation strategies that have access to state feedback or additional ``meta''-training \cite{cesa2006prediction,schapire2012boosting,berend2015finite}. 
	
	This paper proposes a set of non-iterative, unsupervised decision-aggregation strategies with quantifiable performance guarantees. Inspired from the na\"{i}ve Bayes decision rule, which is the likelihood ratio test for known competences $\set{p_1,\cdots,p_N}$, \cite{berend2015finite,poor2013introduction}, we propose block-decision aggregation rules and discuss its adaptive extension. The proposed rules employ biased estimates of expert competences, called \textit{pseudo competences}, directly.   
	\section{Unsupervised Estimation of Competences}\label{SEC:PseudoComp}
	
	Let the experts $\set{X_1,\cdots, X_N}$ be characterized by the probability law \eqref{uniform_prior}-\eqref{opinion_gen}. Then, the true competence $p_i$ of an expert can objectively be measured by the frequency with which the expert successfully identifies the true state:
	\begin{equation}
		\label{emp_est}
		\hat{p}_i(T) = \frac{1}{T} \sum_{t=1}^{T} \ind{X_i(t)=Y(t)}.
	\end{equation}
	The ergodicity of the process $\ind{X_i(t)=Y(t)}$, which follows from \eqref{uniform_prior}-\eqref{opinion_gen}, yields that $\hat{p}_i(T) \rightarrow p_i$ as the number of tasks increases. 
	However, an unsupervised decision maker does not have access to $Y\yay{t}$ and therefore, cannot make use of these reliable competence estimates $\hat{p}_i(T)$ for decision making.  
	
	Conceptually, measuring the quality of an opinion without knowing the true state, or the ground-truth, is a commonly-encountered challenge in human decision-making: One might accept the consensus of extrinsic opinions on a task as the truth to the best of one's knowledge. We define a form of opinion-based reliability, or the \textit{pseudo competence} of an expert, as the likelihood of an expert agreeing with independently-generated opinions from other experts:
	\begin{equation}
		\label{pseudocomp}
		\pp{i} \triangleq \prob{X_i(t) = \V{\Xnot{i}(t)}}, \forall t\in\mathbb{T}.
	\end{equation}
	As discussed in Section \ref{SEC:ProbDef}, majority vote is an intuitive decision rule that is often accepted as the baseline and it leads to the notion of \textit{agreeing with peers}. Formally, the subset $\Xnot{i}$ are the \textit{peers} of the expert $X_i$ and their collective competence under majority vote is denoted by: 
	\begin{equation*}
		\pn{i} \triangleq \prob{Y(t) = \V{\vecnot{X}{i}(t)}}, \forall t\in\mathbb{T}.
	\end{equation*}
	It follows from the law of total probability that pseudo competence is a function of the entire committee $\set{p_1,\cdots,p_N}$ \cite{hajek2015random}:
	\begin{equation}
		\label{pseudo_total_law}
		\pp{i} = p_i\pn{i} + q_i \qn{i},
	\end{equation}
	where, $\qn{i} = 1-\pn{i}$. If one \textit{chooses} to accept the frequency with which an expert agrees with peers as the empirical competence estimate for that expert:
	\begin{equation}
		\label{empiricalpseudo}
		\pest{i}{T} = \frac{1}{T} \sum_{t=1}^{T} \ind{x_i(t) = \V{\xvec{\setminus i}(t)}}
	\end{equation}
	then, it is possible to infer competences, in the form of pseudo competences, in real-time. Similar to the true competence estimates, the ergodicity of the process $\ind{X_i(t) = \V{\Xvec{\setminus i}(t)}}$, which follows \eqref{uniform_prior}-\eqref{opinion_gen}, yields that $\pest{i}{T}\rightarrow \pp{i}$ as the number of tasks increases.  Furthermore, \eqref{empiricalpseudo} enables distributed estimation and ranking of competences on strongly connected networks \cite{sevuktekin2019distributed}. 
	
	The key aspects of the pseudo competence are the exclusion of self-opinions $\Xnot{i}$ and the concomitant competence of the peers $\pn{i}$. Exclusion of self-opinions is rather intuitive as an expert always agrees with itself, hence including self-opinions would bias the pseudo competence towards the expert that is being measured. The collective expertise of the peers on the other hand, is critical for pseudo competence to make sense. Conceptually, as one should not measure the competence of an expert with the likeliness of failure, one should not measure it by the likelihood of agreement with those who fail often. 
	
	Nonetheless, it is clear from \eqref{pseudo_total_law} that $\pn{i} > \nicefrac{1}{2}$, $\forall i\in[N]$ would ensure meaningful inference in the sense that pseudo competence preserves ordering and the sign of $p_i-\nicefrac{1}{2}$. If in addition, $\pn{i}\approx\pn{j}$, $\forall i\neq j$, one could also expect $\abs{\pp{i}-p_i}$ to scale similarly. 
	It is, however, not clear whether an arbitrary committee would satisfy either of these conditions. Next, we address what we call good committees $(p_i> \nicefrac{1}{2}, \forall i\in [N])$ and mixed committees $(p_i \in [0,1], \forall i\in [N])$ and explore the conditions that would yield reliable and meaningful inference.  
	\subsection{Properties of Pseudo Competence for Good Committees}
	We call a committee \textit{good} when every expert is sufficiently competent, $p_i> \nicefrac{1}{2}$, $\forall i\in[N]$ for a finite $N$, and consequently, when every expert has collectively competent peers, $\pn{i}> \nicefrac{1}{2}$, $\forall i\in[N]$. It is often mistaken that for countably-infinite committees $\set{p_i : i\in \mathbb{N}}$ that $p_i>\nicefrac{1}{2}, \forall i$, is sufficient for $\pn{i}>\nicefrac{1}{2}$, $\forall i$; a counter example is given in \cite[pg. 17]{paroush1997stay}. Therefore, for countably-infinite good committees, we allow that $p_i>\nicefrac{1}{2}+\varepsilon$ for some $\varepsilon\in \yay{0,\nicefrac{1}{2}}$ for every expert. Proposition \ref{pseudoprop} summarizes the key properties of pseudo competence for a good committee. 
	\begin{prop}
		\label{pseudoprop}
		For every good committee of size $N>2$ pseudo competences satisfy:
		\begin{enumerate}
			\item Ordering: $p_i>p_j \iff \pp{i}>\pp{j}$,
			\item Under-estimation: $\nicefrac{1}{2}<\pp{i}<p_i$.
		\end{enumerate} 
	\end{prop}
	
	The pseudo competences preserve the ordering of the true competences and they are strictly greater than $\nicefrac{1}{2}$ for good committees -- proof is given in appendix \ref{app:pseudoprop}. On the other hand, pseudo competence penalizes the most competent experts while evaluating lower-competence experts relatively more accurately as illustrated in fig. \ref{fig:pseudo_good} for $N=10$ experts with competences uniformly spaced over $[0.5,0.9]$. Even though a good committee is sufficient to ensure that committee as a whole is sufficient in the absence of any individual expert $\yay{\pn{i}>\nicefrac{1}{2},\forall i}$, it is not necessary. Therefore, we discuss the conditions that ensure reliable peers for every expert for mixed committees next. 
	\begin{figure}[t!]
		\centering
		\begin{subfigure}[t]{0.24\textwidth}
			\centering
			\includegraphics[width=\textwidth]{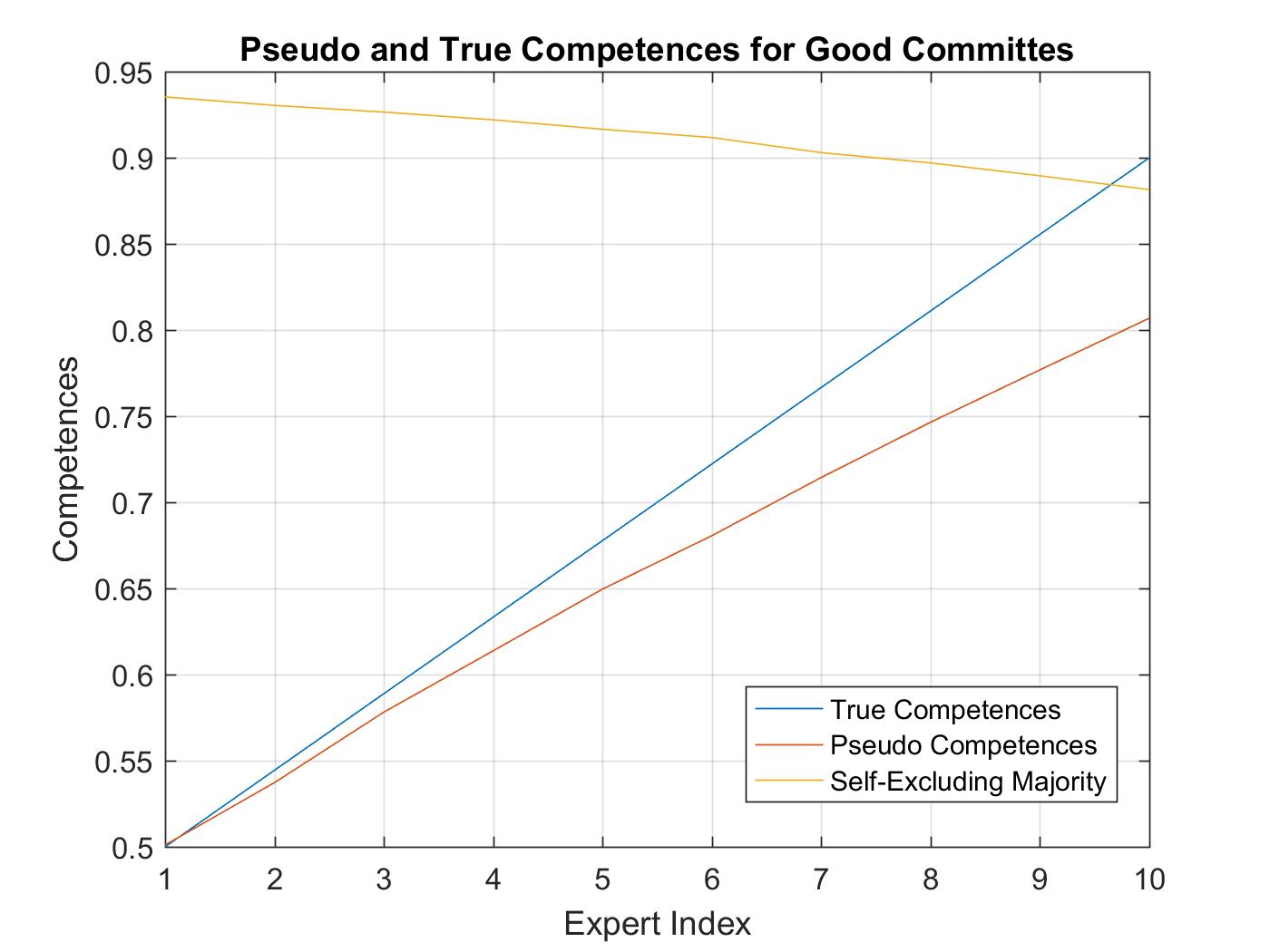}
			\caption{A Good Committee}
			\label{fig:pseudo_good}
		\end{subfigure}%
		~ 
		\begin{subfigure}[t]{0.24\textwidth}
			\centering
			\includegraphics[width=\textwidth]{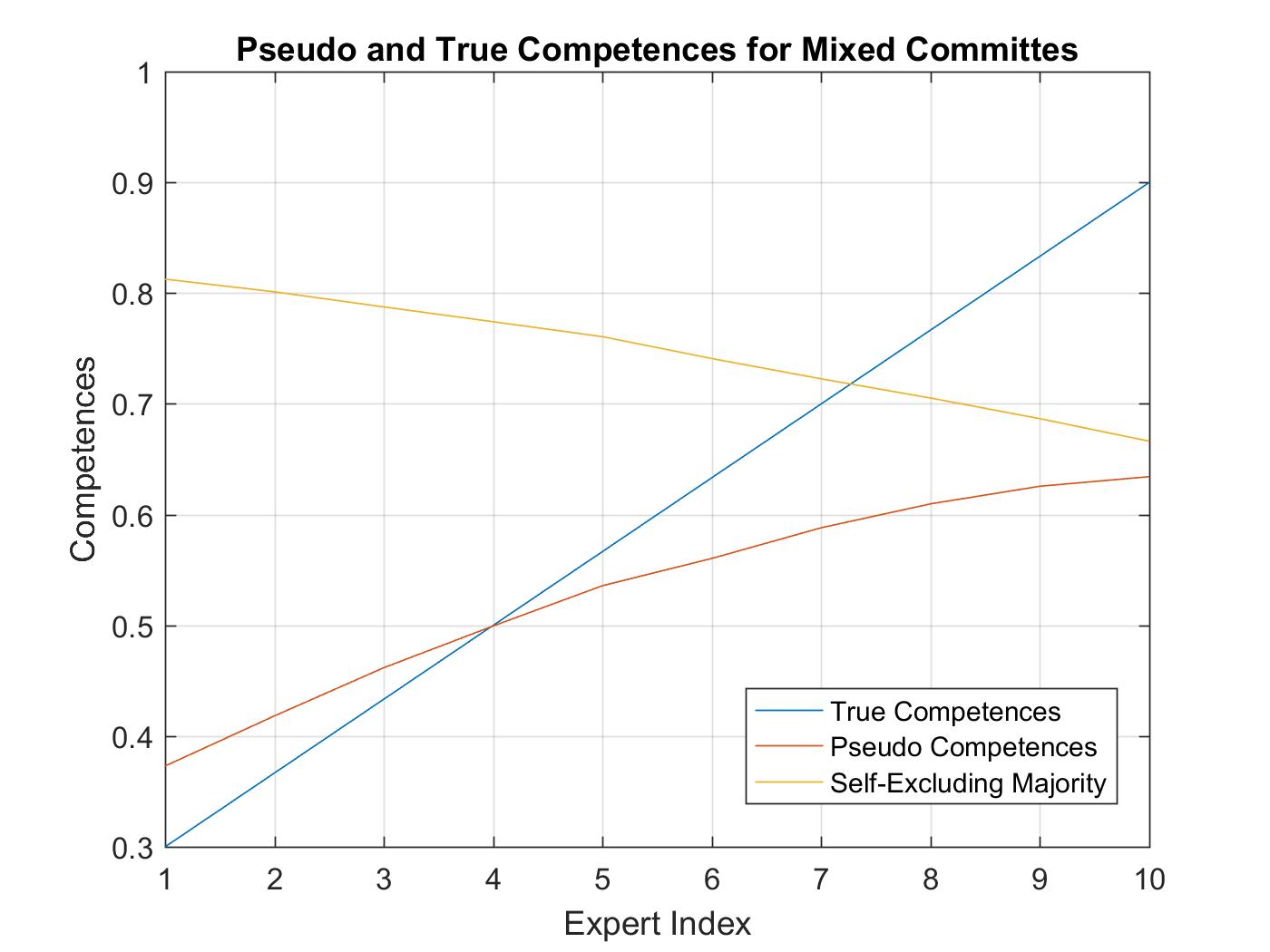}
			\caption{A Mixed Committee}
			\label{fig:pseudo_mixed}
		\end{subfigure}
		\caption{A Comparison of Pseudo Competences $\pp{i}$ to True Competences $p_i$, Self-Excluding Majority $\pn{i}$ as Reference}
		\label{fig:pseudo_overall}
	\end{figure}
	\subsection{Properties of Pseudo Competence for Mixed Committees}
	Pseudo competence relies on the notion that the committee, as a whole, is a sufficiently competent reference point that it is robust to the absence of any individual expert. Certainly, this notion is not valid for all mixed committees: Reliable unsupervised opinion aggregation is unattainable when many experts are unreliable, for instance, $\yay{\abs{i: p_i\leq \nicefrac{1}{2}}\rightarrow N}$. Nonetheless, it is of practical interest to investigate mixed committees that are competent as a whole.
	
	Often, a notion of \textit{consistency} arise in various supervised or unsupervised learning, inference, and decision rules i.a. \cite{cesa2006prediction}. Consistency under majority vote, also known as \textit{Condorcet's Jury Theorem} \cite{berend1998condorcet,paroush1997stay,marquis1785essai}, is defined as follows.
	\begin{definition}[Consistency]
		\label{consistency}
		A committee of experts $\set{X_i}_{i\in\mathbb{N}}$ with competences $\set{p_i}_{i\in\mathbb{N}}$ is consistent under majority voting if:
		\begin{equation*}
			\lim_{N\rightarrow\infty}\prob{\V{X_1,\cdots,X_N}=Y}=1.
		\end{equation*}
	\end{definition} 
	An explicit characterization of consistent mixed committees is non-trivial and addressed in \cite{paroush1997stay,berend1998condorcet,ben2000nonasymptotic}. 
	One should note that the sequence $\prob{\V{X_1,\cdots,X_N}=Y}$ is not necessarily monotonically increasing in $N$ for a consistent committee, even when $p_{i+1}>p_i$, $\forall i$.
	%can prove to be rather counter-intuitive. For instance, adding more experts does not always increase $v_n$, as illustrated by any committee with competences  $\set{1,p,p,\cdots}$ for some $p\in(\nicefrac{1}{2},1)$. Even a committee with ordered competences $p_1<p_2<\cdots$ does not ensure that $v_n$ is monotonic. 
	However, for every consistent committee, there exists a monotonically increasing, tight lower bound $a_N\leq v_N$ that we call the \textit{rate} of consistency: 
	\begin{equation}
		\label{rate}
		a_N = \inf_{n\geq N} \prob{\V{X_1,\cdots,X_n}=Y}.
	\end{equation}
	The rate of consistency determines the asymptotic performance of decision aggregation rules that use pseudo competences. %It is clear that no sharper monotonic lower bound exists. 
	Proposition \ref{pseudopropmixed} extends the notion of under-estimation in proposition \ref{pseudoprop} to what we call \textit{pessimistic} estimation for mixed committees:
	
	\begin{prop}
		\label{pseudopropmixed}
		For every consistent committee $\set{p_i}_{i\in\mathbb{N}}$, there exists a committee size $n^{*}$ such that $\forall N>n^{*}$ pseudo competences satisfy:
		\begin{enumerate}
			\item Ordering: $p_i \geq p_j \iff \pp{i}\geq \pp{j}$
			\item Pessimism: $\min\lb \pp{i},1-\pp{i}\rb \geq \min\lb p_i,1-p_i\rb$.
		\end{enumerate}
	\end{prop}
	Conceptually, it indicates that good experts $\yay{p_i>\nicefrac{1}{2}}$ are underestimated and bad ones $\yay{p_j<\nicefrac{1}{2}}$ are overestimated -- proof given in appendix \ref{app:mixedpseudo}. Figure \ref{fig:pseudo_mixed} illustrates proposition \ref{pseudopropmixed} for $N=10$ experts with competences uniformly spaced over $\brac{0.3,0.9}$, where pessimism property is illustrated visually. This property arises as more competent experts in a fixed committee are compared to the relatively weaker subsets of peers that exclude those competent experts. A condition for \textit{finite} mixed committees to satisfy proposition \ref{pseudopropmixed} is given by \eqref{pseudocompetencediff} in appendix \ref{app:pseudoprop}. 
	
	The pseudo competence \eqref{pseudocomp} provides a metric that can be estimated in real-time via \eqref{empiricalpseudo} and that requires no side-information beyond the existence of a probability law characterized by \eqref{uniform_prior}-\eqref{opinion_gen}. It preserves the ordering of the true competences and it neither mistakes a bad expert for a good one nor it does the opposite. Therefore, it provides a reliable unsupervised metric for estimating the underlying probability law directly. Section \ref{SEC:Naive} investigates the minimum probability of error achieving opinion-aggregation strategy, then section \ref{SEC:Pseudo_naive} substitutes the true probability law with the pseudo probability law.   

	\section{Na\"{i}ve Bayes Decision Rule and Its Performance Guarantees}\label{SEC:Naive}
	%gghg
	Consider the family $\mathcal{D}$ of possibly randomized decision rules that aggregate $N$ opinions $\rvec{X}\yay{t}$ to decide on a state $Y(t)$. The minimum probability of error (MPE)-achieving strategy is an instantaneous opinion-aggregation rule:
	\begin{equation*}
		f^{MPE} = \argmin_{f\in\mathcal{D}} \prob{f\yay{\rvec{X}}\neq Y}.
	\end{equation*}
	Note that the competences $\set{p_1,\cdots,p_N}$ being fixed across all tasks $t\in\mathbb{T}$ allows task-dependency of opinions $\rvec{X}\yay{t}$ and states $Y(t)$ to be dropped. It is well known that the maximum a posteriori rule achieves the minimum probability of error \cite{poor2013introduction}. Therefore, the MPE-achieving supervised opinion-aggregation strategy is given by:
	\begin{equation*}
		f^{MPE}\yay{\rvec{X}} = \argmax_{y\in\set{\pm 1}} \condprob{y}{\rvec{X}},
	\end{equation*}
	with the corresponding likelihood-ratio test:
	\begin{equation}
		\label{naivebayes}
		f^{MPE}(\mathbf{X}) = \sign\left(\sum_{i=1}^N X_i \log\frac{p_i}{q_i}\right) \equiv f^{NB}(\mathbf{X}).
	\end{equation} 
	The decision rule in \eqref{naivebayes} is often referred to as \textit{na\"{i}ve Bayes decision rule} (NB) \cite{berend2015finite}. The existence of the probability law defined by \eqref{uniform_prior}-\eqref{opinion_gen} ensures that NB decision rule is instantaneous and that it has no \textit{bias term} in the sign function. 
	
	As a direct consequence of being the maximum a posteriori rule, the probability of error of the na\"{i}ve Bayes decision rule can be written explicitly as:  
	\begin{align*}
		\prob{f^{NB}(\Xvec{})\neq Y} 
		%&=\sum_{\dvec{x}\in\set{\pm 1}^{N}}\condprob{f^{NB}(\Xvec{})\neq Y}{\dvec{x}}\prob{\dvec{x}}
		%\\ 
		&= \frac{1}{2}\sum_{\dvec{x}\in\set{\pm 1}^{N}} \min_{y\in\set{\pm 1}}\condprob{\Xvec{}=\dvec{x}}{Y=y}. 
	\end{align*}
	%Note that we have used a concise notation for the conditionals\footnote{Due to space constraints, we wrote $\condprob{\cdot}{x}\equiv \condprob{\cdot}{X=x}$ for conditionals. In addition, when a random value is evaluated at a certain point we wrote $\condprob{Y=y}{\cdot} \equiv \condprob{y}{\cdot}$, collectively leading to $\condprob{Y=y}{X=x} \equiv \condprob{y}{x}.$}. 
	Even though $\min_{y\in\set{\pm 1}}\condprob{\Xvec{}=\dvec{x}}{Y=y}$ is easy to compute $\forall \dvec{x}$ when $\set{p_1,\cdots,p_N}$ is known, the sum is still intractable for large $N$, which motivates research for lower and upper bounds. It has been shown that a set of lower and upper bounds can be found in the form \cite[Theorem 1]{berend2015finite}:
	\begin{equation*}
		-\log\prob{f^{NB}(\Xvec{})\neq Y} \asymp \Phi,
	\end{equation*} 
	where $\asymp$ denotes upper and lower bounds within a constant factor. As a function of the true competences $\set{p_1,\cdots,p_N}$, $\Phi$ is called the \textit{committee potential}, \cite{berend2015finite,berend2013sharp,kearns1998large} and it is given by:
	\begin{equation}
		\label{potential}
		\Phi\yay{\dvec{p}} = \sum_{i=1}^{N}\left(p_i-\frac{1}{2}\right)\log\frac{p_i}{q_i}.
	\end{equation}
	The upper-bound given in \cite[Theorem 1(i)]{berend2015finite} makes a clever use of the Chernoff bounding technique, see, for instance, \cite[Chapter 2.2.1]{raginsky2013concentration}, and Kearns-Saul inequality \cite[Lemma 1]{kearns1998large}. A detailed discussion on Kearns-Saul and Berend-Kontorovich concentration inequalities for mixtures of independent, bounded random variables is given in \cite[Chapter 2.2.4]{raginsky2013concentration}. 
	\begin{figure}[t!]
		\centering
		\includegraphics[width=.45\textwidth]{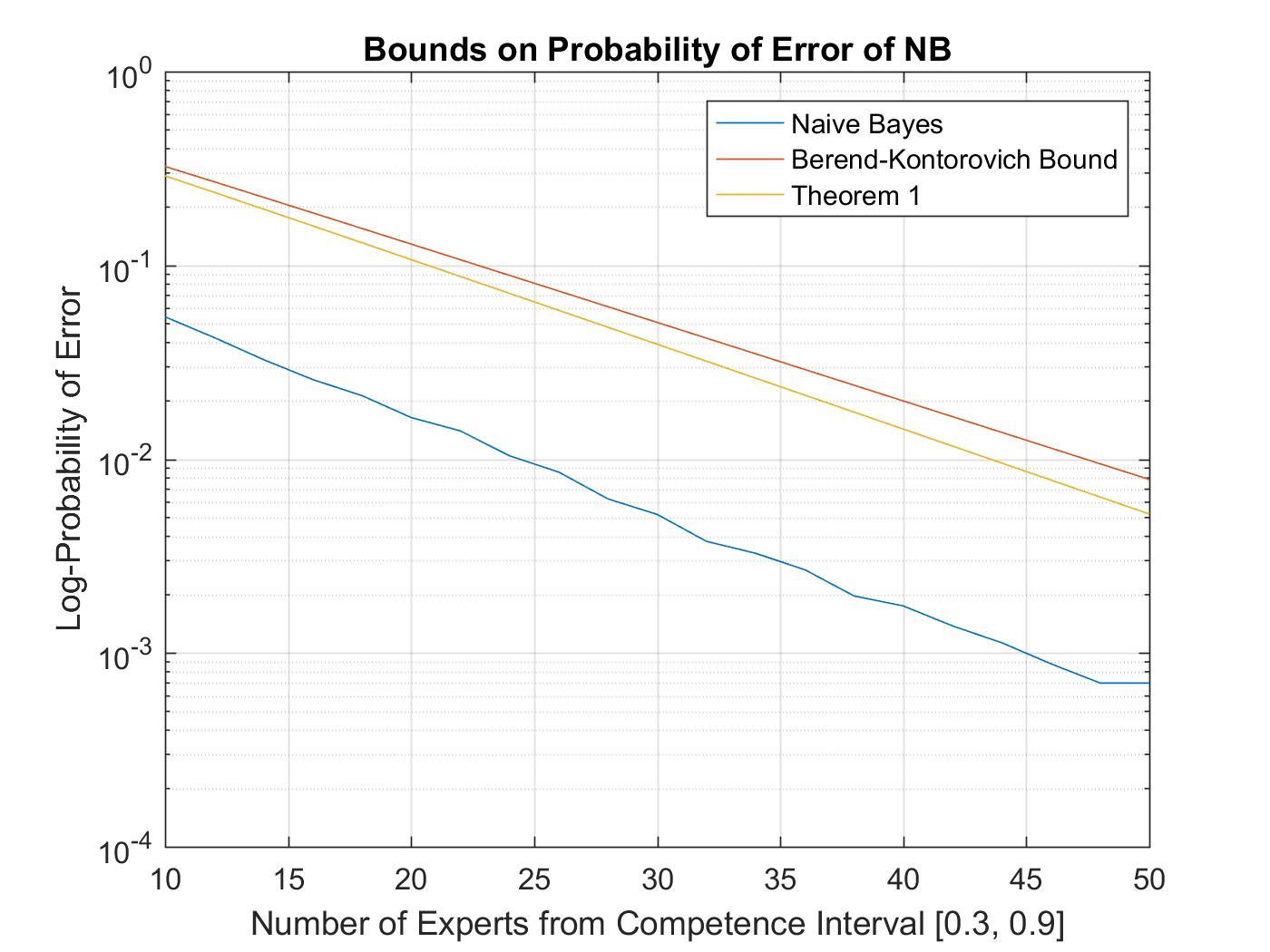}
		\caption{A Comparison of Upper-Bounds for Na\"{i}ve Bayes Probability of Error}
		\label{fig:bound_comparison}
	\end{figure}
	
	Interestingly, in the case of na\"{i}ve Bayes decision rule, the subsequent use of Kearns-Saul inequality appears unnecessary. As shown in appendix \ref{app:thm_improved_NB}, a direct consequence of the Chernoff bounding technique is as follows:
	\begin{thm}
		\label{improved_upper}
		The following tight upper-bound holds:
		\begin{equation*}
			\prob{f^{NB}\yay{\rvec{X}} \neq Y} \leq \prod_{i=1}^{N}\sqrt{4 p_i q_i}.
		\end{equation*}
	\end{thm}
	Theorem \ref{improved_upper} is the sharpest bound attainable by the use of Chernoff bounding technique and it is tight, as evidenced by, $\exists i: p_i=1$, in addition to the regimes discussed in \cite{berend2015finite}. Figure \ref{fig:bound_comparison} provides a comparison of bounds between that given in theorem \ref{improved_upper} and that in \cite[Theorem 1]{berend2015finite} for $N\in\brac{10,50}$ experts with equally-spaced competences chosen from the interval $\brac{0.3, 0.9}$. 
	
	Observe that $w(p)=\log\frac{p}{1-p}\in\yay{-\infty,\infty}$ is an unbounded function that monotonically increases over $p\in[0,1]$. On one hand, when competences are known reliably, $w(p)$ ensures that highly competent experts are almost exclusively relied upon. On the other hand, in empirical setups, where competences are to be estimated, direct use of $w(\hat{p})$ would jeopardize the robustness of opinion-aggregation rule: An expert with arbitrary competence $p_i<1$ could be assigned, with non-negligible probability, a competence estimate $\hat{p}_i\approx 1$ and thus, an unbounded weight when the number of labeled tasks is limited. In such cases, a \textit{linearized} na\"{i}ve Bayes (LNB) decision rule could be considered:
	\begin{equation}
		\label{linearnaivebayes}
		f^{LNB}(\mathbf{X}) = \sign\left(\sum_{i=1}^N X_i \yay{p_i-\nicefrac{1}{2}}\right).
	\end{equation} 
	Alternatively, one might focus on a practical subset of committees, one that we call \textit{absolutely balanced} committees, that excludes almost-all-knowing $\yay{p\approx 1}$ and almost-ever-lying $\yay{p\approx 0}$ experts.
	\begin{definition}[Absolute Balance]
		\label{absolutebalance}
		A committee $\set{p_i}_{i\in\mathbb{N}}$ is absolutely-balanced if $\exists \gamma \in \left(0, \nicefrac{1}{2}\right)$ such that $p_i\in(\gamma,1-\gamma)$, $\forall i$. The factor $\gamma$ is called the balancing parameter of the committee.
	\end{definition}
	For an absolutely-balanced committee, not only the true competence estimates $\set{\hat{p}_1(T),\cdots, \hat{p}_N(T)}$, defined in \eqref{emp_est}, yield robust empirical implementation of the NB rule through direct substitution into \eqref{naivebayes} but they also enable the empirical use of LNB rule by substitution into \eqref{linearnaivebayes} that exhibit similar performance with NB. 
	
	Section \ref{SEC:PseudoComp} shows that a set of opinions can be used to estimate true competences $\set{p_1,\cdots,p_N}$ in the form of pseudo competences $\set{\pp{1},\cdots,\pp{N}}$. Section \ref{SEC:Pseudo_naive} shows that for an absolutely balanced committee, one could indeed achieve an unsupervised opinion-aggregation performance that scales with the committee potential, therefore, with the performance of MPE-achieving rule, while suffering from error due to bias between pseudo and true competences that \textit{diminishes} in the number of experts consulted.
	\section{The Pseudo Na\"{i}ve Bayes Decision Rule}
	\label{SEC:Pseudo_naive}
	In the unsupervised problem setup, true competence estimates $\set{\hat{p}_1,\cdots,\hat{p}_N}$ are not achievable, where, over $T$ tasks, the pseudo competence estimates $\set{\pest{1}{T},\cdots,\pest{N}{T}}$ can be inferred. An opinion-aggregation rule that use pseudo competence estimates does not only suffer performance degradation due to empirical estimation error but due to the bias $\abs{p_i-\pp{i}}$ as well. The goal of this section is to first propose fundamental limits for a decision rule that has access to pseudo competences $\set{\pp{1},\cdots,\pp{N}}$ \textit{directly}.
	
	Let us call the following opinion aggregation rule the \textit{pseudo na\"{i}ve Bayes} (PNB) rule: 
	\begin{equation}
		\label{pseudonaivebayes}
		f^{PNB}(\mathbf{X}) = \sign\left(\sum_{i=1}^N X_i \log\frac{\pp{i}}{\pq{i}}\right).
	\end{equation} 
	The PNB rule corresponds to \textit{assuming} that the underlying probability law is characterized by the pseudo competences $\set{\pp{1},\cdots,\pp{N}}$, where it is actually characterized by the true competences $\set{p_1,\cdots,p_N}$. PNB rule, similar to NB rule, is an instantaneous decision rule. Different from NB rule, PNB rule is empirically achievable without supervision.  
	
	Formally, we show that the performance of PNB rule scales with the underlying true committee potential $\Phi\yay{\dvec{p}}$, as defined in \eqref{potential}, and the performance degradation due to the use of pseudo competences is quantified as follows: 
	\begin{equation*}
		-\log\prob{f^{PNB}(\Xvec{})\neq Y} \asymp (1-\delta(a_N,\gamma))\Phi.
	\end{equation*}
	Here, $\delta(\cdot)$ represents the performance degradation due to lack of supervision and it is a bounded function of the rate of consistency $a_N$ from \eqref{rate} and the balancing parameter $\gamma$ from definition \ref{absolutebalance}. The variable $\delta(\cdot)$ diminishes both in $a_N$ and $\gamma$ due to the difference between pseudo competences and true competences diminishing in $a_N$ and $\gamma$ limiting the maximum difference between pseudo competences and true competences. 
	
	Let us note that the performance of the PNB decision rule exhibits a similar scaling to that of NB, as long as exclusion of any expert leaves committee reliable ($\pn{i}>\nicefrac{1}{2}$, $\forall i$):
	\begin{equation}
		\label{performancePNB}
		-\log\prob{f^{PNB}(\Xvec{})\neq Y} \asymp \tilde{\Phi}.
	\end{equation}
	Here, we call the term $\tilde{\Phi}$ pseudo committee potential:
	\begin{equation*}
		\tilde{\Phi} = \sum_{i=1}^{N}\left(p_i-\frac{1}{2}\right)\log\frac{\pp{i}}{1-\pp{i}}.
	\end{equation*}
	The relation in \eqref{performancePNB} follows algebraically from the proof of \cite[Theorem 1]{berend2015finite} by the use of property \ref{pseudopropmixed},
	%and extends further to general weighted mixtures, 
	and it is outlined in appendix \ref{app:PNBproof}. 
	
	The pseudo committee potential $\tilde{\Phi}$ and true committee potential $\Phi$ converge at a rate that is determined by the rate at which the underlying committee becomes consistent in the majority vote. Theorem \ref{limitperformance} quantifies the rate at which the performance of the PNB decision rule scales with that of NB.  
	\begin{thm}
		\label{limitperformance}
		Every absolutely balanced, consistent committee with rate $a_N$ and balancing parameter $\gamma$ satisfies:
		\begin{equation*}
			\frac{\tilde{\Phi}}{\Phi}\geq 1- C(\nicefrac{1}{2}-\gamma)\rho_N \longrightarrow_N 1.
		\end{equation*}
		Here, $\rho_N=\nicefrac{\yay{1-a_N}}{\yay{a_N-\nicefrac{1}{2}}}$ and $C(x)$ is a positive function supported on $x\in[-\nicefrac{1}{2},\nicefrac{1}{2}]$.
	\end{thm}
	Theorem \ref{limitperformance} indicates that PNB rule is not only asymptotically optimal, but it approaches to the performance of the optimal (supervised) decision rule that becomes consistent at a rate that is faster than that of majority vote, $a_N$. A dual of this result can be formulated as follows:  
	\begin{corollary}
		\label{fixedsizepotential}
		A sufficient condition for any absolutely balanced committee of size $N$ to ensure that $1-\delta \leq \frac{\tilde{\Phi}}{\Phi}$ is as follows:
		\begin{equation*}
			\frac{1-p_{\setminus i}}{p_{\setminus i}-\nicefrac{1}{2}} \leq \frac{\delta}{C(\nicefrac{1}{2}-\gamma)}, \forall i\in[N].
		\end{equation*} 
	\end{corollary}
	Decision rules that employ some relevant statistics in a functional form that is known to perform well are often called ``plug-in'' decision rules. Plug-in rules are often difficult to analyze and are often suboptimal \cite[Chapter 1]{haykin2005adaptive}. Interestingly, despite being a plug-in rule, PNB rule achieves minimum probability of error asymptotically, thanks to the construction of the pseudo competences.
	
	PNB rule, as defined in \eqref{pseudonaivebayes}, is instantaneous because it uses pseudo competences directly, which, in practice, should be estimated from the opinions empirically. The manner in which pseudo competence estimates are updated gives rise to the operational meaning of the corresponding decision aggregation rule, as addressed in Section \ref{SEC:EmpiricalPseudoNaiveBayes}.   
	%%%%%%%%%%%%%%%%%%%%%%%%%%%%%%%%%%%%%%%%%%%%%%%%%%%%%%%%%%%%%%%%%%%%%%%%%%%%%%%%%%%%%%%%%%%%%%%
	%Here, states and opinions $Y, X_i\in\set{-1,1}$ almost always and since true competences are assumed to be static, the task dependence $(Y(t),\mathbf{X}(t))$ is only emphasized when they have contextual meaning. 
	%
	%The performance of a decision aggregation rule is measured by its ability to identify the hidden state correctly:
	%\begin{equation*}
	%\prob{f^{PNB}(\mathbf{X})=Y}.
	%\end{equation*}
	%%
	%Pseudo competence $\pp{i}$ of an expert is measured with respect to the opinions of all other available experts; $\Xvec{\setminus i} \triangleq \set{X_1,..,X_{i-1},X_{i+1},...,X_{N}}$. The bias $\left|p_i-\pp{i}\right|$ depends on the competence of the underlying committee under majority voting, denoted by:
	%\begin{equation*}
	%\pn{i} \triangleq \prob{\V{\Xvec{\setminus i}} =Y},
	%\end{equation*}
	%which directly determines the performance of PNB rule for a fixed committee. 
	%
	%
	%
	%
	%\subsection{Performance Limits for Pseudo Na\"{i}ve Bayes Decision Rule}
	%\label{SEC:PseudoNaiveBayes}
	%
	%The performance of the NB decision rule scales with what is called \textit{committee potential} \cite{berend2015finite}:
	%%
	%\begin{equation}
	%%\label{NBperf}
	%-\log\prob{f^{NB}(\Xvec{})\neq Y} \asymp \Phi.
	%\end{equation}
	%%
	%%
	%Here, $\asymp$ denotes equivalence up to universal multiplicative constants and for a given competence profile $\pvec{}$ and $\Phi$ denotes the \textit{true} committee potential: 
	%%
	%%
	%\begin{equation*}
	%
	%
	%
	\section{Empirical Rules that Use Pseudo-Competences}
	\label{SEC:EmpiricalPseudoNaiveBayes}
	
	%PNB decision rule as formulated in \eqref{pseudonaivebayes} is an instantaneous, static, and \textit{supervised} decision rule, quite similar to original NB decision rule from \eqref{naivebayes}. However, it extends naturally to unsupervised setup since it is possible to estimate pseudo competences empirically via:
	%%
	%\begin{equation}
	%\label{empiricalpseudo2}
	%\pp{i}(T) = \frac{1}{T} \sum_{t=1}^{T} \ind{x_i(t) = \V{\xvec{\setminus i}(t)}} \longrightarrow_{T} \pp{i},
	%\end{equation}
	%where the convergence in $T$ follows from the ergodicity of the random process $\ind{X_i(t) = \V{\Xvec{\setminus i}(t)}}$ for static true competences $\bm{p}$. 
	A block-processing, empirical, and unsupervised opinion-aggregation rule can be constructed from pseudo competence estimates via:
	\begin{equation}
		\label{empiricalpseudonaivebayes}
		f^{B}(\mathbf{X}(t)) = \sign\left(\sum_{i=1}^N \dvec{X}_i(t) w(\pest{i}{T})\right).
	\end{equation} 
	Note that \eqref{empiricalpseudonaivebayes} first estimates pseudo competences over $T$ tasks \textit{then} aggregates the available opinions. With this understanding, one might drop the time-dependence for a block-processing rule. The weight function $w(\cdot)$ can be picked based on the number of tasks. 
	% \eqref{empiricalpseudonaivebayes} is written in a block decision aggregation form, where 
	The adaptive extension of \eqref{empiricalpseudonaivebayes} is as follows: 
	\begin{equation}
		\label{adaptivepseudonaivebayes}
		f^{A}_{t}(\mathbf{X}(t)) = \sign\left(\sum_{i=1}^N X_i(t) w(\pest{i}{t})\right).
	\end{equation} 
	Asymptotically, both of the empirical rules should achieve optimality at a rate close to that of NB rule, as Theorem \ref{limitperformance} indicates. The objective is to quantify the impact of empirical estimation error.
	%
	%%
	%
	%The strength of a committee is a function of the competences $\bm{p}$ of its experts; conceptually, the higher the individual competences, the stronger is the committee. Proposition \ref{pseudoprop} indicates that as long as the committee of opinion sources consists entirely of experts, meaning that $p_i>\nicefrac{1}{2}$ $\forall i$, pseudo competence preserves ordering of the true competences by a strict underestimation. One might postulate that unless comprising the majority, sources with true competence $p=\nicefrac{1}{2}$, or gamblers, and those with $p<\nicefrac{1}{2}$, or adversaries, should not affect the properties of pseudo competence and hence, PNB should remain a valid approach for certain mixed committee profiles (those comprising experts, adversaries, and gamblers) as well. Section \ref{SEC:PropertiesPseudoCompetence} lays the foundations for this notion. 
	
	\subsection{Block Opinion Aggregation}
	Pseudo competences can be estimated dynamically in real-time or over a block of opinions. When a block of opinions is processed, the performance of the corresponding empirical PNB decision rule $f^{B}\yay{\cdot}$ is determined by:
	\begin{enumerate} 
		\item \textit{Pseudo Competence Bias:} $\norm{\ppvec{}-\dvec{p}}_1$,
		\item \textit{Empirical Estimation Error:} $\norm{\pestvec{T}-\ppvec{}}_1$.
	\end{enumerate}
	Note that pseudo competence bias is the cost of operating in an unsupervised setup and it is a hidden-yet-fixed function of $\dvec{p}$ for every committee. Empirical estimation error, on the other hand, introduces a non-linear distortion that propagates through the weights $w\yay{\pestvec{T}}$ of the decision rule.  
	
	When each expert can only be consulted for a small number $T$ of tasks, a rather coarse, high-variance, estimate, $\pp{i}(T)$, is achievable, and hence the corresponding weights %should be modified either by employing a different function rather than one such as 
	$w(\pp{})= \log\frac{\pp{}}{1-\pp{}}$ might be arbitrarily large. %or use a variant of the Laplacian estimator. 
	In order to rectify this non-robust behavior, a \textit{linearized} version of the weights $\ell(\pp{})= \pp{}-\nicefrac{1}{2}$, which follows from the first order Taylor series expansion of $w\yay{\pp{}}$ around $\pp{}=\nicefrac{1}{2}$, can be used:
	\begin{equation}
		\label{lightload}
		f^{L}(\mathbf{X}) = \sign\left(\sum_{i=1}^N X_i (\pest{i}{T}-\nicefrac{1}{2})\right).
	\end{equation} 
	It is clear that \eqref{lightload} should not suffer from estimation error as much as  empirical PNB rule with $w(\pp{})= \log\nicefrac{\pp{}}{\pq{}}$ \eqref{empiricalpseudonaivebayes}. The main challenge for such a rule is to achieve consistency, as addressed next.
	
	\begin{prop}
		\label{lowsamplingconsistency}
		If a committee satisfies:
		\begin{enumerate}
			\item (Na\"{i}ve Bayes) $\lim_{N\rightarrow\infty} \frac{1}{\sqrt{N}}\sum_{i=1}^{N}(p_{i}-\nicefrac{1}{2})^2 = \infty$,
			\item (Majority) $\lim_{N\rightarrow\infty} \frac{1}{\sqrt{N}}\sum_{i=1}^{N}(p_{i}-\nicefrac{1}{2}) \geq \sqrt{\frac{\log 2}{8}}$,
		\end{enumerate}
		then $f^L(\cdot)$ is consistent: $\lim\limits_{N\rightarrow \infty}\prob{f^{L}(\Xvec{})\neq Y}=0$. 
	\end{prop}
	Proposition \ref{lowsamplingconsistency} shows that when arbitrarily large number of experts are consulted for a small number of tasks each, the unsupervised empirical PNB rule becomes reliable. 
	
	The first condition of proposition \ref{lowsamplingconsistency} is a direct consequence of the Hoeffding's inequality ans it is sufficient for consistency when each expert is \textit{tested} with a small number of \textit{labeled} tasks and competences are estimated via \eqref{emp_est} \cite[Theorem 7]{berend2015finite}. Interestingly, pseudo competences can facilitate this consistency without the need for labeled data but at the expense of the second condition, which
	%\begin{equation}
	%f^{L}(\mathbf{X}(t)) = \sign\left(\sum_{i=1}^N X_i(t) (\hat{p}_i(T)-\nicefrac{1}{2})\right)
	%\end{equation} 
	%One should observe that the second condition of proposition \ref{lowsamplingconsistency} 
	is by no means restrictive: As discussed in appendix \ref{app:prop_consistency_low}, it amounts to $\lim_{N\rightarrow \infty} \prob{\V{\dvec{X}}=Y}>\nicefrac{1}{2}$. Similar requirements exist in the Bayesian framework as well \cite{ok2016optimality}. 
	%
	%
	%
	%An alternative take to the use of approximate weights is to modify the empirical estimation rule to eliminate high probability $\tilde{w}_{i}\in\set{-\infty,\infty}$ events using a Laplacian estimator of the form:
	%%
	%\begin{equation}
	%\label{laplacian}
	%\pp{i}^{\ell}(T) = \frac{1}{T+2} \left[\sum_{t=1}^{T} \ind{x_i(t) = \V{\xvec{\setminus i}(t)}}+1\right].
	%\end{equation}
	%%
	%This inherently creates an absolutely balanced sequence with the understanding that $p_i\in\left(\frac{1}{T+2},\frac{T+1}{T+2}\right)$, $\forall i\in[N]$. 
	%
	
	When there are sufficiently many tasks to be processed, the error due to empirically estimating $\ppvec{}$ diminishes and pseudo competence bias becomes the dominant factor of performance degradation.  Thus, the PNB rule given in \eqref{pseudonaivebayes} naturally extends to a decision rule \eqref{empiricalpseudonaivebayes} that empirically estimates the pseudo competence of each expert over $T$ tasks and applies to $w\yay{\pp{i}}=\log\nicefrac{\pp{i}}{\pq{i}}$.
	%
	%, therefore the impact of estimation error can be quantified and incorporated into the error due to decision aggregation.   
	%
	%
	%\subsubsection{Empirical PNB Rules for Heavy Task-Load}
	%
	%
	%Under heavy task-loads, 
	%
	%\begin{equation}
	%\label{heavyload}
	%f^{}(\mathbf{X}(t)) = \sign\left(\sum_{i=1}^N X_i(t) \frac{\pp{i}(T)}{1-\pp{i}(T)}\right).
	%\end{equation} 
	%\textbf{HERE}
	%Such a decision aggregation rule has two interpretations: If the decisions are made $\forall t\leq T$ \textit{after} $\pp{i}(T)$ is inferred, then it is a block processing, static rule. If the first $T$ tasks are declared to be a \textit{learning phase} and the resulting rule is employed $\forall t>T$ then it is a static instantaneous rule. In both cases, the performance of $f^{H}(\cdot)$ is a function of two different estimation errors:
	%\begin{enumerate} 
	%	\item \textit{Pseudo Competence Bias:} $\norm{w(\ppvec{})-w(\bm{p})}_1\equiv \norm{\tilde{\bm{w}}-\bm{w}}_1$
	%	\item \textit{Empirical Estimation Error:} $\norm{\tilde{\bm{w}}(T)-\tilde{\bm{w}}}_1$
	%\end{enumerate}
	%This is a recurrent theme with the use of pseudo competences; what is originally an empirical estimation problem often becomes a joint estimation and inference problem due to lack of supervision. 
	
	For an arbitrary committee $\dvec{p}\in\brac{0,1}^{N}$, the difference between pseudo competences and true competences become unbounded for experts with $p_i\approx \set{0,1}$. However, for absolutely balanced committees, pseudo competence bias is necessarily bounded.  Lemma \ref{deviation} quantifies the committees that limit the difference between PNB and NB weights.
	\begin{lemma}
		\label{deviation} 
		Let $R(\gamma) = \frac{2\gamma (1-\gamma)}{1-2\gamma}$, if an absolutely balanced committee satisfies for some $\epsilon >0$ 
		\begin{equation*}
			\min_{i\in[N]} \pn{i}\geq \frac{1}{2} + \frac{1}{2+ \epsilon R(\gamma)},
		\end{equation*}
		then $\norm{\bm{w}-\tilde{\bm{w}}}_1\leq \frac{\epsilon N}{2}$.
	\end{lemma}
	Observe $R(\gamma)$ increases in $\gamma$, equivalently, committees $\dvec{}$ that concentrate towards the center of the cube $\brac{0,1}^{N}$ yield closer weights. Conceptually, this amounts to discussion on weak classifiers; it is often easier to boost weak classifiers to form a stronger one \cite{schapire2012boosting}. The next theorem jointly addresses the empirical estimation error and pseudo competence bias:
	\begin{thm}
		\label{highsampling}
		Let a committee be consistent with rate $a_N$,  $\forall \delta \in \left(0,1\right)$ define $C(\delta;N,T) \triangleq \frac{12}{T}\log\frac{8N}{\delta}$. Then, 
		\begin{equation*}
			\forall \epsilon \in \left(\left(\rho_N C(\delta;N,T)\right)^{\nicefrac{1}{3}},\min\lb 5, \frac{2\Phi}{N}\rb\right),
		\end{equation*}
		and for all absolutely balanced committees with parameter $\gamma >  C(\delta;N,T)\left(\frac{2}{\sqrt{4\epsilon+1}-1}\right)^2$:
		\begin{equation*}
			\prob{f^{B}(\mathbf{X})\neq \rvec{Y}} \leq \delta + \exp\left[-\frac{\left(2\Phi -\epsilon N\right)^2}{8\Phi}\right].
		\end{equation*}
	\end{thm}
	Property \ref{pseudopropmixed} along with Lemma \ref{deviation} allows $f^{B}(\cdot)$ to scale similar to an empirical NB as long as the underlying worker committee is sufficiently strong, which is captured by $\delta$. Theorem \ref{highsampling} borrows its empirical analysis from that of \cite[Theorem 11]{berend2015finite}, which quantifies the performance of empirical NB under sufficiently long training. 
	
	Albeit insightful, Proposition \ref{lowsamplingconsistency} and Theorem \ref{highsampling} analyze the performance of empirical PNB decisions rules for a \textit{block} of opinions. An adaptive and instantaneous opinion aggregation rule is addressed next.
	\subsection{Adaptive Opinion Aggregation}
	\label{SEC:Realtime}
	Let $f^{A}_{\tau}(\cdot)$ denote the empirical pseudo na\"{i}ve Bayes decision rule defined by the weights $\tilde{\bm{w}}(\tau)$ as given in \eqref{adaptivepseudonaivebayes}.
	%, that is, $f^{A}_{\tau}(\cdot)\equiv f^{H}(\Xvec{}^{\tau})$
	%
	%\footnote{These decision rules are only equivalent in the sense that they both use $\tilde{\bm{w}}(\tau)$ as the mixing coefficients: $f^{A}_{\tau}(\Xvec{}(t))$ instantaneously aggregates decisions of $\Xvec{}(t)$ for $t>\tau$, where $f^{H}(\Xvec{}^{\tau})$ aggregates the decisions of $\Xvec{}(t)$ for $\forall t<\tau$.}. 
	We call the probability that the decision rule $f^{A}_{\tau}(\cdot)$ makes the correct decision based on $X(t)$, that is, $\prob{f^{A}_{\tau}(\Xvec{}(t))=Y}$, for $\forall t>\tau$, the \textit{confidence} of the real-time decision rule. Theorem \ref{online} characterizes this notion of confidence:
	\begin{thm}
		\label{online}
		Let $\delta \geq \sum_{i=1}^{N}\left|p_i-\pp{i}\right|+\frac{N}{\sqrt{\tau}}$ and define the event $R(\tau)$ as follows:
		\begin{equation*}
			\exp\left(-\frac{1}{2}\sum_{i=1}^{N}\left(\pp{i}(\tau)-\frac{1}{2}\right)\log\frac{\pp{i}(\tau)}{1-\pp{i}(\tau)}\right) \leq \frac{\delta}{2}.
		\end{equation*}
		Then $\forall t>\tau$, $\prob{R(\tau)\cap \lb f^{A}_{\tau}\left(\mathbf{X}(t)\right) \neq Y \rb} \leq \delta$.
	\end{thm}
	The term $\sum_{i=1}^{N}\left|p_i-\pp{i}\right| \leq \sum_{i=1}^{N}(1-p_{\setminus i})$, diminishes with the committee potential, which can be seen from Kearns-Saul or Berend-Kontorovich inequalities.
	
	The analysis of \cite{berend2015finite} on the adaptive empirical na\"{i}ve Bayes decision rule is based on the committee potential being empirically estimated from some labeled data to control the worst case performance on the test data. Theorem \ref{online} extends this analysis to use empirical pseudo competences instead, resulting in a real-time algorithm where the player builds an empirical pseudo committee potential and makes decisions with dynamic confidence.
	
	The adaptive decision aggregation rule $f^{A}_{\tau}(\cdot)$ is a sequential decision making mechanism: at any given time $t\in[T]$ the algorithm makes a decision with confidence $\delta$ if $R(\tau)$ has happened for some $\tau <t$ otherwise, it assigns the majority vote. This allows the algorithm to a make decisions with dynamic confidence; once a confidence level is achieved at $t=\tau$, there is no need to keep updating the weights as the decision rule is finalized and algorithm uses that fixed decision rule on the incoming data. 
	\section{Experiments}\label{SEC:Exp}
	\begin{figure}[t!]
		\centering
		\begin{subfigure}[t]{0.24\textwidth}
			\centering
			\includegraphics[width=\textwidth]{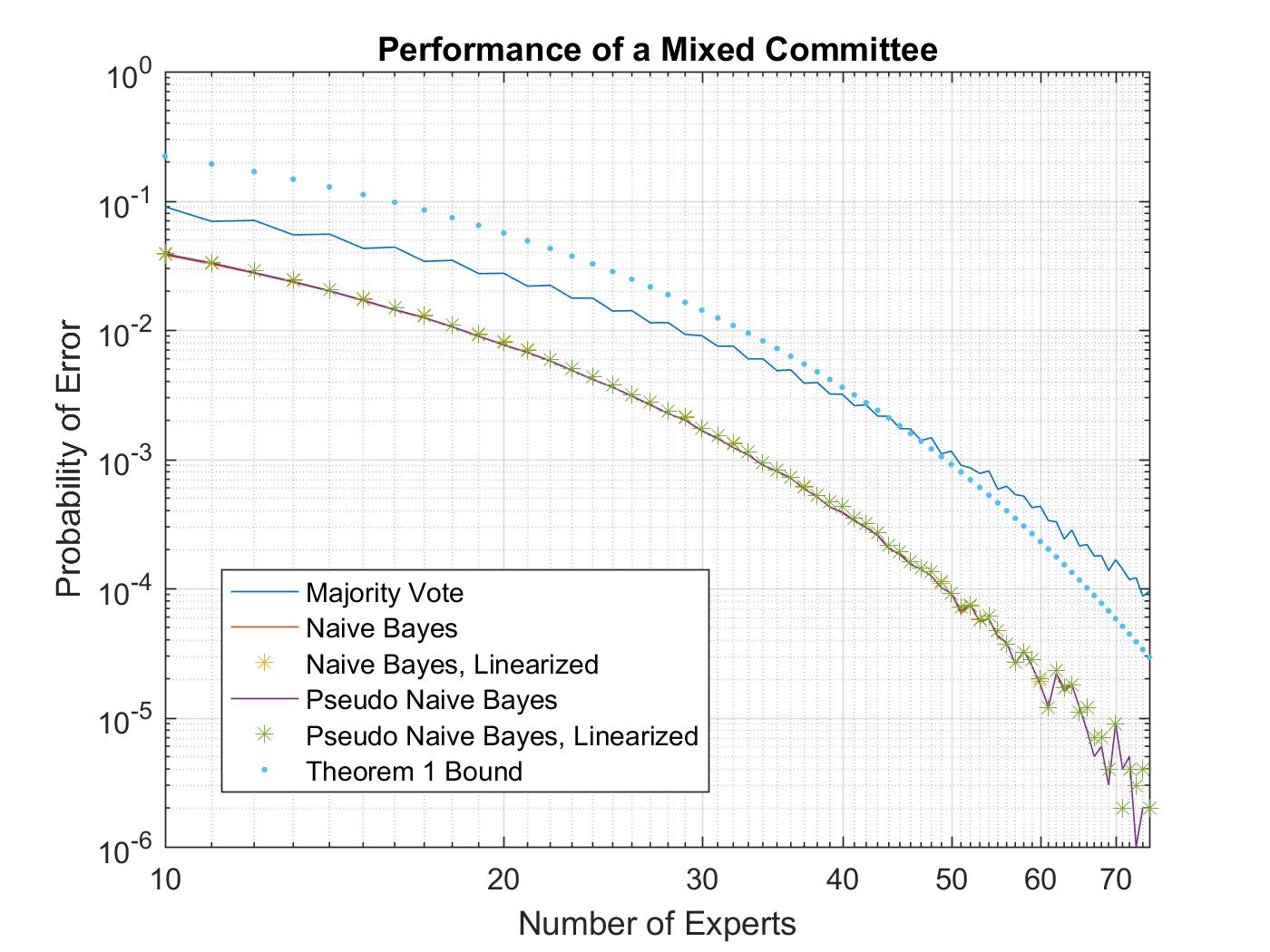}
			\caption{A Good Committee}
			\label{fig:pnb_good}
		\end{subfigure}%
		~ 
		\begin{subfigure}[t]{0.24\textwidth}
			\centering
			\includegraphics[width=\textwidth]{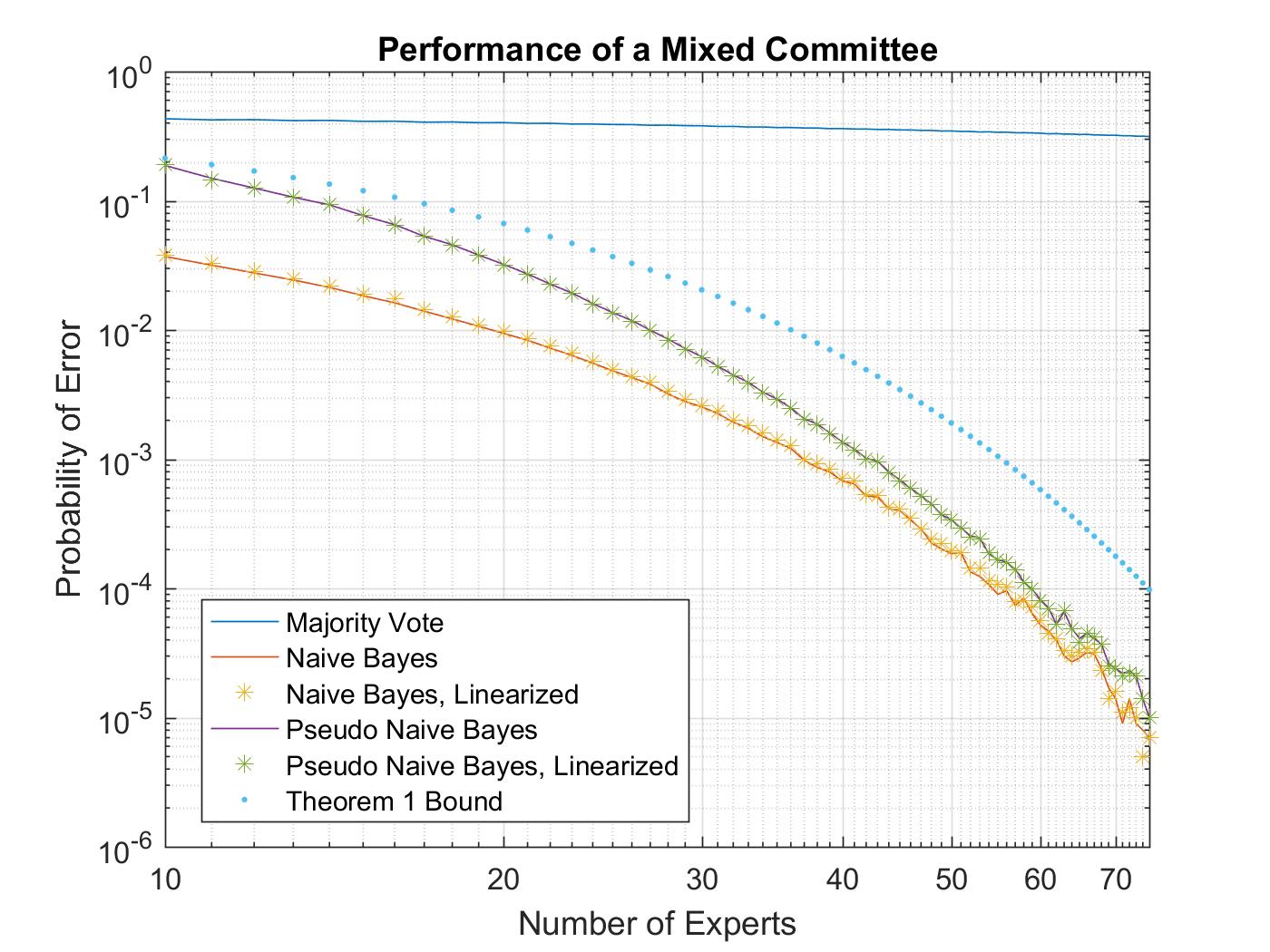}
			\caption{A Mixed Committee}
			\label{fig:pnb_mixed}
		\end{subfigure}
		\caption{A Performance Comparison between Na\"{i}ve Bayes Rules and Their Unsupervised Counterparts}
		\label{fig:pnb_overall}
	\end{figure}
	We briefly illustrate through synthetic data the impact of mixed/good committees of varying sizes $N$ and balancing parameters $\gamma$ on the performance of pseudo na\"{i}ve Bayes decision rule.   
	
	In our experiments, pseudo-random tasks and experts that satisfy \eqref{uniform_prior}-\eqref{opinion_gen} are generated on MATLAB. In order to compare the performance of na\"{i}ve Bayes decision rule to its linearized and unsupervised counterparts, we consider a committee of experts with sizes varying from $N=10$ to $N=75$ with competences equally spaced in the intervals $[0.15, 0.9]$ for the mixed-committee case and from $[0.5, 0.9]$ for the good-committee case. Figure \ref{fig:pnb_overall} compares the performance of NB to that of PNB, as defined in \eqref{pseudonaivebayes}, where pseudo competences are estimated over $T=1e+6$ tasks. 
	
	The objective is to observe the fundamental operational tendencies: In the case good committees, illustrated in fig. \ref{fig:pnb_good}, with the majority vote probability of error starting at the level of $0.1$, the performance difference between PNB and NB are negligible, which is to be expected since $p_i>\nicefrac{1}{2}$ for every expert. In the case of a mixed committee, illustrated in fig. \ref{fig:pnb_mixed}, where the majority vote probability of error varies in the interval $\yay{0.4,0.37}$ for every committee size, the performance difference with PNB and NB rules is evident. Furthermore, in the case of a mixed committee the performance PNB converges to that of NB due to the robustness of the performance of majority vote over large populations against individual expert exclusions.
	\begin{figure}[t!]
		\centering
		\begin{subfigure}[t]{0.24\textwidth}
			\centering
			\includegraphics[width=\textwidth]{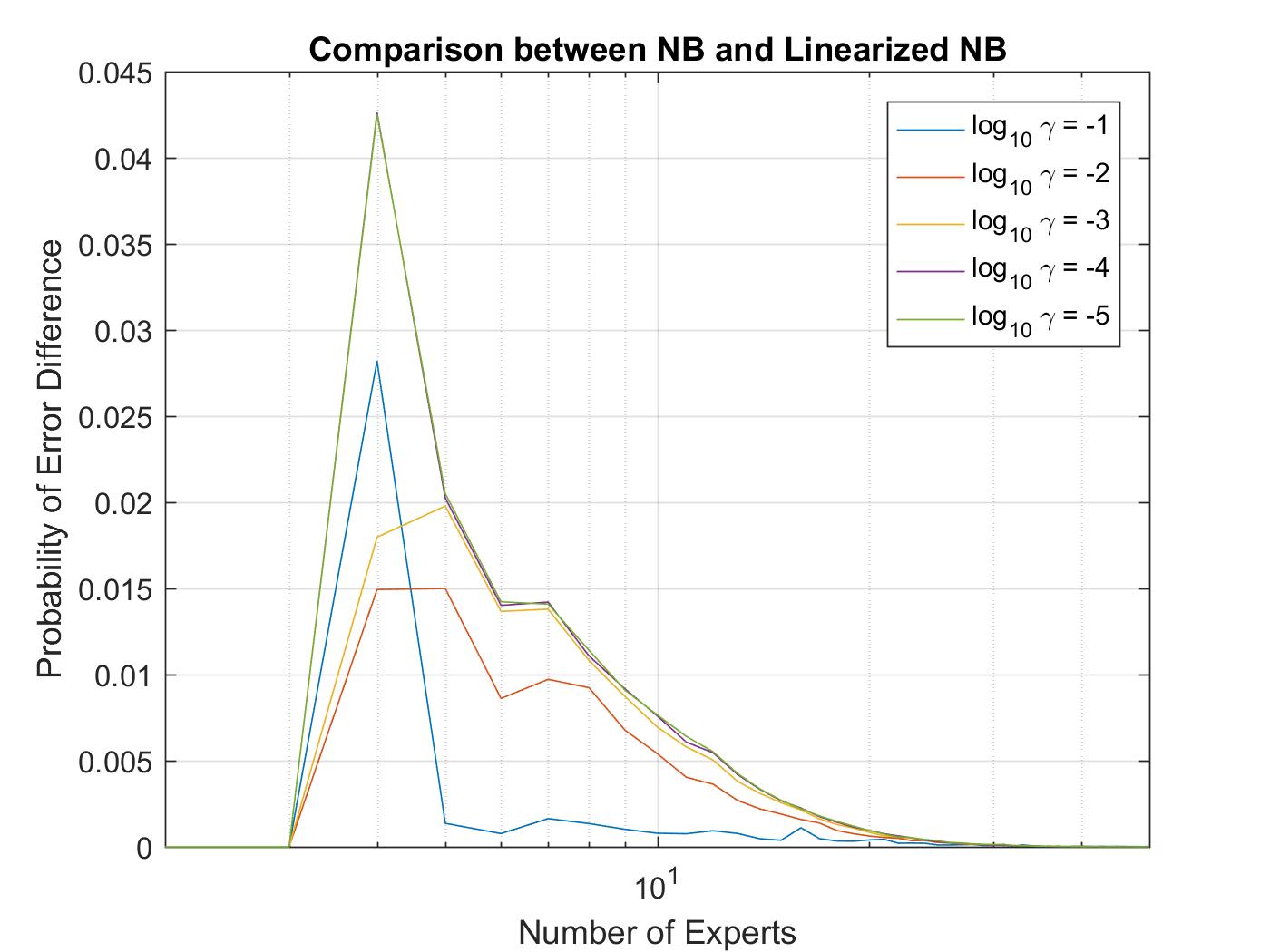}
			\caption{NB vs. LNB}
			\label{fig:nb_vs_lnb}
		\end{subfigure}%
		~ 
		\begin{subfigure}[t]{0.24\textwidth}
			\centering
			\includegraphics[width=\textwidth]{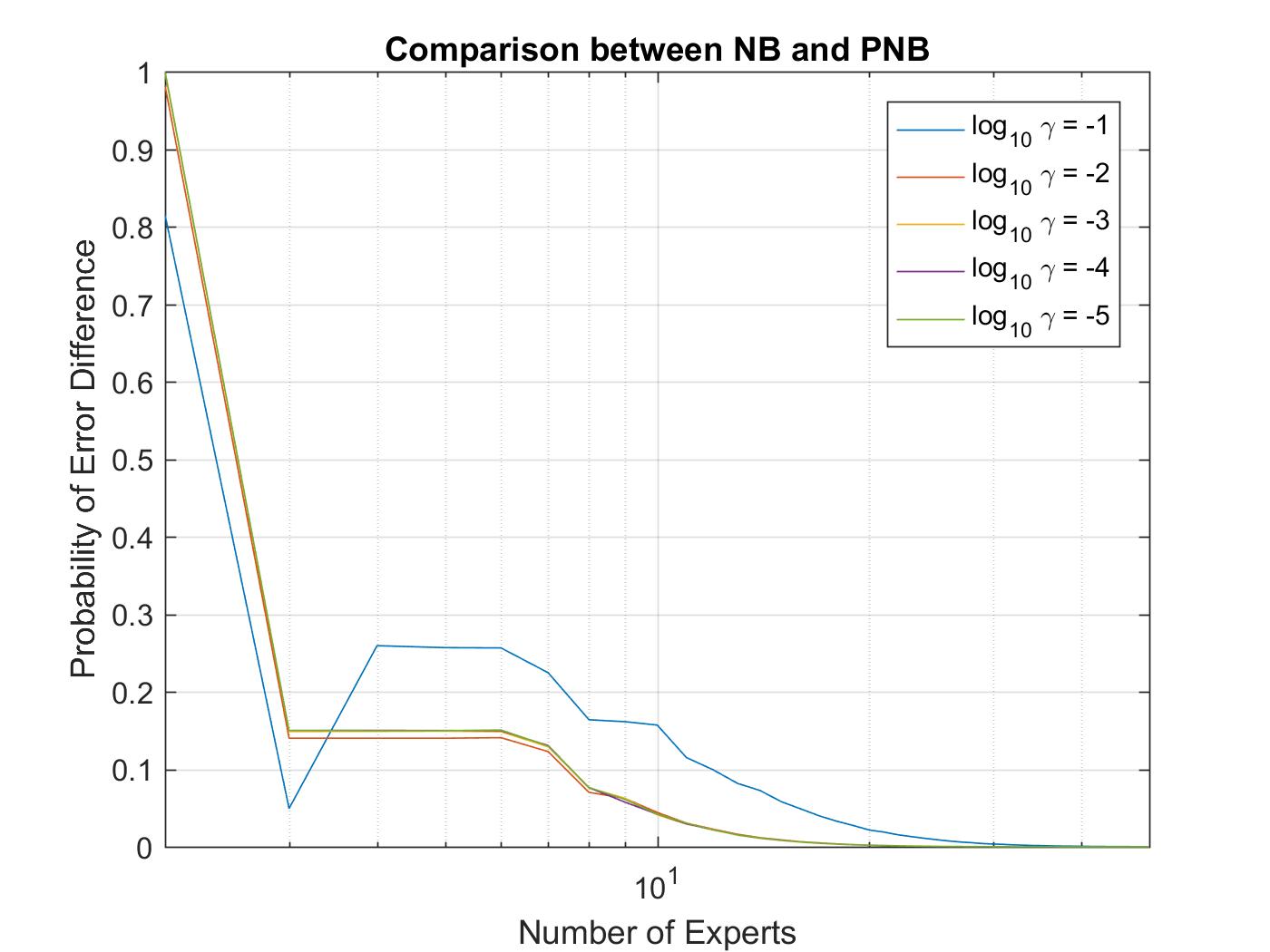}
			\caption{NB vs. PNB}
			\label{fig:nb_vs_pnb}
		\end{subfigure}
		\begin{subfigure}[t]{0.24\textwidth}
			\centering
			\includegraphics[width=\textwidth]{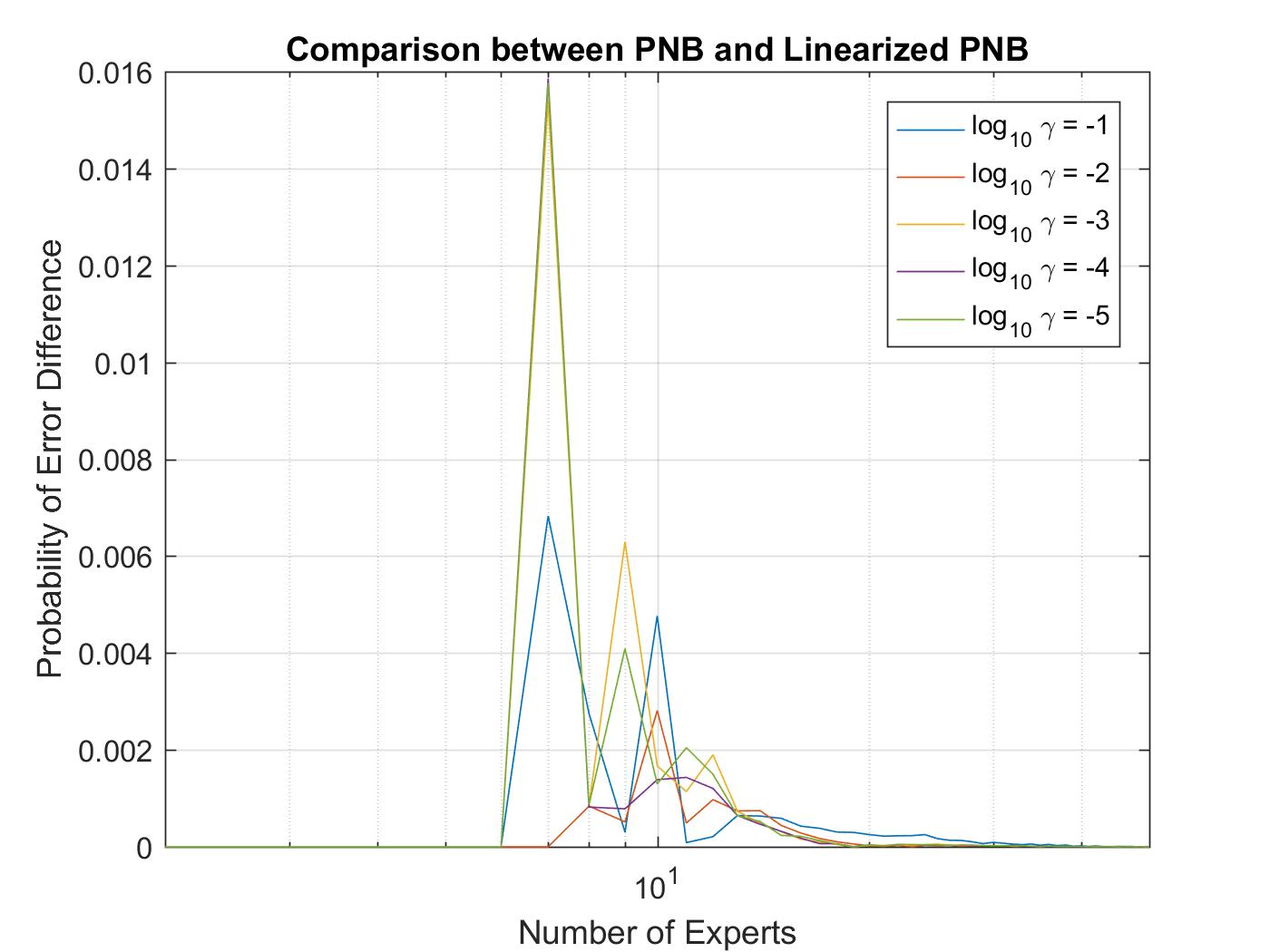}
			\caption{PNB vs. Linearized PNB}
			\label{fig:pnb_vs_lpnb}
		\end{subfigure}%
		~ 
		\begin{subfigure}[t]{0.24\textwidth}
			\centering
			\includegraphics[width=\textwidth]{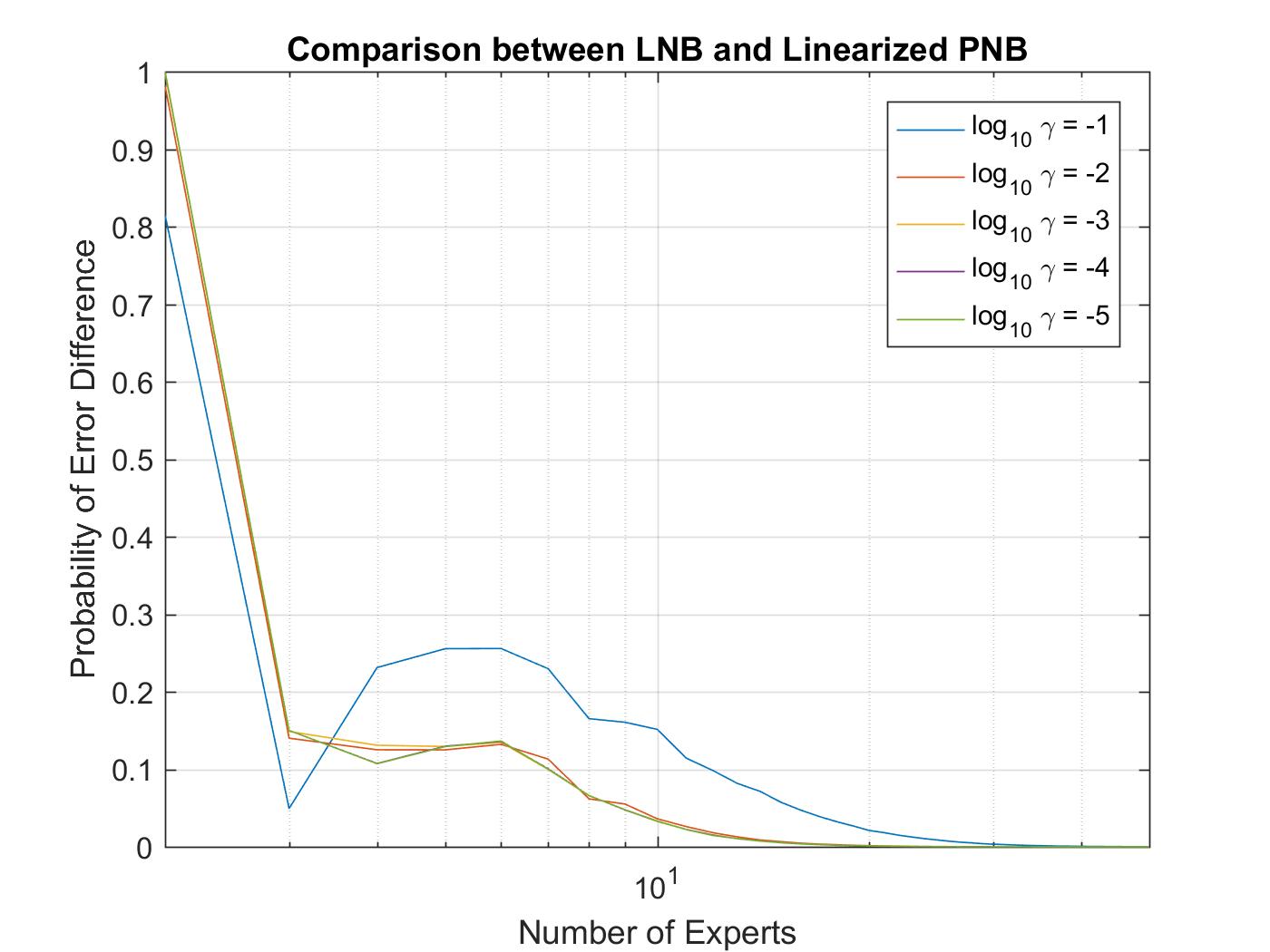}
			\caption{LNB vs. Linearized PNB}
			\label{fig:lnb_vs_lpnb}
		\end{subfigure}
		\caption{Impact of Balancing Parameter $\gamma$ on Performance Degradation}
		\label{fig:balancing}
	\end{figure}
	
	Figure \ref{fig:balancing} illustrates the impact of the balancing parameter $\gamma =1e-n$ for $n\in\set{1,\cdots,5}$. We allow competences to be picked equally spaced from $\brac{0.5,1-\gamma}$ for committee sizes $N\in\set{2,\cdots,50}$. The performance degradation with respect to that of NB rule, for LNB, as seen in fig. \ref{fig:nb_vs_lnb}, and for PNB, as seen in fig. \ref{fig:nb_vs_pnb}, increases as $\gamma$ decreases. This is to be expected as the probability of error for NB decision rule is necessarily bounded by $\gamma$, which is unattainable by other rules due to weight-linearization and pseudo competence estimation. 
	\appendices
	%%%%%%%%%%%%%%%%%%%%%%%%%%%%%%%%%%%%%%%%%%%%%%%%%%%%%%%%%%%%%%%%%%%%%%%%%%%%%%%%%%%%%%%%%%%%%%%%%%%%%%%%%%%%%%%%%
	\section{Proof of Proposition \ref{pseudoprop}}
	\label{app:pseudoprop}
	\begin{proof}[Proof of Part (1)-Ordering]
		Consider any pairs of experts $(X_i,X_j)$ for $i\neq j$, and allow $\eta_i = \ind{X_i = Y}$. Observe that $\eta_i \perp \eta_j, \forall i\neq j$ are Bernoulli random variables with parameter $p_i$, denoted by $\mathcal{B}(p_i)$. Successive application of the law of total probability yields that:
		\begin{align*}
			\pp{i} 
			&= \prob{X_i = \V{\Xnot{i}}} = \sum_{\eta_i}\condprob{X_i = \V{\Xnot{i}}}{\eta_i}\prob{\eta_i}  \\
			&= \sum_{\eta_i, \eta_j} \condprob{X_i = \V{\Xnot{i}}}{\eta_i,\eta_j}\condprob{\eta_j}{\eta_i}\prob{\eta_i} .
			%&= p_i\prob{\V{\Xnot{i}}=Y} + q_i\prob{\V{\Xnot{i}}\neq Y} \\ 
			%&= p_ip_j \condprob{\V{\Xnot{i}}=Y}{X_j=Y} + p_iq_j \condprob{\V{\Xnot{i}}=Y}{X_j\neq Y} \\
			%&+q_ip_j \condprob{\V{\Xnot{i}}\neq Y}{X_j=Y} + q_iq_j \condprob{\V{\Xnot{i}}\neq Y}{X_j\neq Y}
		\end{align*}	
		Observe that $\condprob{\eta_j}{\eta_i} = \prob{\eta_j}$ due to the conditional independence of opinions (hence the independence of opinion \textit{generation} $\eta_i$). A similar extension of $\pp{j}$ yields that:
		\begin{align*}
			\pp{j} = \sum_{\eta_i, \eta_j} \condprob{X_j = \V{\Xnot{j}}}{\eta_i,\eta_j}\prob{\eta_i}\prob{\eta_j}.
			%&= p_jp_i \condprob{\V{\Xnot{j}}=Y}{X_i=Y} + p_jq_i \condprob{\V{\Xnot{j}}=Y}{X_i\neq Y} \\
			%&+q_jp_i \condprob{\V{\Xnot{j}}\neq Y}{X_i=Y} + q_jq_i \condprob{\V{\Xnot{j}}\neq Y}{X_i\neq Y}
		\end{align*}	
		Since the rest of the committee is arbitrary, yet fixed, the following conditional probabilities are equal: 
		\begin{align*}
			\condprob{\V{\Xnot{i}}=Y}{X_j=Y} &=\condprob{\V{\Xnot{j}}=Y}{X_i=Y}, \\
			\condprob{\V{\Xnot{i}}\neq Y}{X_j\neq Y} &=\condprob{\V{\Xnot{j}}\neq Y}{X_i\neq Y},
		\end{align*}
		and $p_iq_j - p_jq_i = (p_i-p_j)$. Therefore, the ratio between the differences between pseudo competences and true competences are given by:
		\begin{align}
			%& =\nn \\
			&\nicefrac{\pp{i}-\pp{j}}{p_i-p_j} = \label{pseudocompetencediff} \\
			%(p_iq_j-p_jq_i) \condprob{\V{\Xnot{j}}=Y}{X_i\neq Y} \\ +(q_ip_j-q_jp_i)\condprob{\V{\Xnot{j}}\neq Y}{X_i=Y} \nn \\
			%= \yay{p_i-p_j}\times \\ 
			&\brac{\condprob{\V{\Xnot{j}}=Y\!}{X_i\neq Y} -\condprob{\V{\Xnot{j}}\neq Y\!}{\!X_i=Y}}. \nn 
		\end{align}
		As long as \eqref{pseudocompetencediff} is positive, pseudo competence preserves ordering. We now show that \eqref{pseudocompetencediff} is monotonically increasing in $p_i\in\yay{\nicefrac{1}{2},1}$, $\forall i$ and the minimum is zero. Observe that total law of probability yields by \eqref{bigequation} that $\forall n\in[N]$:
		%	\begin{figure*}[t]
			%		\begin{equation}
				%		\label{bigequation}
				%		\frac{\partial}{\partial p_k} \prob{\sum_{i=1}^{N}\eta_i \geq n} 
				%		\!= \frac{\partial}{\partial p_k} \sum_{l=n}^{N} \sum_{\substack{\mathcal{I}\subset [N]\\ \abs{\mathcal{I}}=l}} \prod_{i\in\mathcal{I}}p_i \prod_{i\in\mathcal{I}^c}q_i = \sum_{l=n}^{N} \sum_{\substack{\mathcal{I}\subset [N]\\ \abs{\mathcal{I}}=l}} \left(\ind{k\in\mathcal{I}}\!\prod_{i\in\mathcal{I}\setminus k}\!\!p_i \prod_{i\in\mathcal{I}^c}q_i - \ind{k\in\mathcal{I}^c}\prod_{i\in\mathcal{I}}p_i \!\!\prod_{i\in\mathcal{I}^c\setminus k}\!\!q_i\right) 
				%		\end{equation}
			%	\end{figure*}
		\begin{figure*}[t]
			\begin{equation}
				\label{bigequation}
				\frac{\partial}{\partial p_k} \prob{\sum_{i=1}^{N}\eta_i \geq n} 
				= \frac{\partial}{\partial p_k} \!\brac{p_k \prob{\sum_{i\neq k}\eta_i \geq n-1} \!+ q_k \prob{\sum_{i\neq k}\eta_i \geq n}} 
				\!= \prob{\sum_{i\neq k}\eta_i \geq n-1}\!-\prob{\sum_{i\neq k}\eta_i \geq n}
			\end{equation}
		\end{figure*}
		%	Let $\mathcal{I}^c$ denote the set complement of $\mathcal{I}$, then \eqref{bigequation} yields that, $\forall n$:
		%\begin{equation*}
		%= \sum_{\substack{\mathcal{I}\subset [N]\\ \abs{\mathcal{I}}=n}} \prod_{i\in\mathcal{I}}p_i \prod_{i\in\mathcal{I}^c}q_i \geq 0, \forall n.
		%\end{equation*}
		%Here, the last equality follows combinatorially. Similarly, it follows that:
		%	\begin{align}
			%	\frac{\partial}{\partial p_k} \prob{\sum_{i=1}^{N}\eta_i \geq n} &=  \sum_{\substack{\mathcal{I}\subset [N]\setminus k\\ \abs{\mathcal{I}}=n-1}} \prod_{i\in\mathcal{I}}p_i \prod_{i\in\mathcal{I}^c}q_i > 0,  \label{derivative_result1} \\
			%	\frac{\partial}{\partial p_k} \prob{\sum_{i=1}^{N}\eta_i \leq n} &=  -\!\!\!\!\sum_{\substack{\mathcal{I}\subset [N]\setminus k\\ \abs{\mathcal{I}}=n-1}} \prod_{i\in\mathcal{I}}q_i \prod_{i\in\mathcal{I}^c}p_i < 0, \label{derivative_result2} 
			%	\end{align}
		\begin{align}
			\frac{\partial}{\partial p_k} \prob{\sum_{i=1}^{N}\eta_i \geq n} &= \prob{\sum_{i\neq k}\eta_i = n-1} \geq 0 \label{derivative_result1} \\
			\frac{\partial}{\partial p_k} \prob{\sum_{i=1}^{N}\eta_i \leq n} &= - \prob{\sum_{i\neq k}\eta_i = n} \leq 0 \label{derivative_result2} 
		\end{align}
		%	with the understanding that $\mathcal{I}: \abs{\mathcal{I}}\leq 0$ is the empty set. The equality between \eqref{bigequation} and \eqref{derivative_result1}-\eqref{derivative_result2} is non-trivial: For \eqref{derivative_result1}, among all possible subsets of $\mathcal{I}\subset [N]\setminus k$, only those with cardinality $n-1$ remain after successive cancellation.
		
		Consequently, $\condprob{\V{\Xnot{j}}=Y}{X_i\neq Y}$ decreases in $p_i$, where $\condprob{\V{\Xnot{j}}\neq Y}{X_i=Y}$ increases, $\forall i$. Therefore,
		\begin{align*}
			\min_{\substack{\mathbf{p} \\ p_i> \nicefrac{1}{2}}} \frac{\pp{i}-\pp{j} }{p_i-p_j} 
			%(\condprob{\V{\Xnot{j}}=Y}{X_i\neq Y} -\condprob{\V{\Xnot{j}}\neq Y}{X_i=Y}) \\ 
			= \lim_{\mathbf{p}\rightarrow\nicefrac{\mathbf{1}}{2}} \frac{\pp{i}-\pp{j} }{p_i-p_j} =0,
			%\brac{\condprob{\V{\Xnot{j}}=Y}{X_i\neq Y} \right. \\ &\left.-\condprob{\V{\Xnot{j}}\neq Y}{X_i=Y}} =0,
		\end{align*}
		which yields that $\pp{i}>\pp{j} \iff p_i > p_j$, if $p_i>\nicefrac{1}{2}$, $\forall i$.
	\end{proof}
	Conceptually, \eqref{derivative_result1}-\eqref{derivative_result2} dictate that increasing the competence of an expert necessarily decreases the probability of error for majority vote. One should note that this does not contradict the discussion in Section \ref{SEC:PseudoComp} as the consistency of majority vote is concerned with \textit{adding} a new expert, which does \textit{not} ensure any monotonicity, instead of \textit{increasing} the competence of an expert, which, as shown, \textit{does} ensure monotonicity.
	%%%%%%%%%%%%%%%
	%%%%%%%%%%%%%%%
	%%%%%%%%%%%%%%%
	\begin{proof}[Proof of Part (2)]
		Observe that $\pn{i}> \nicefrac{1}{2}$ for all finite good committees with $p_i>\nicefrac{1}{2}$, $\forall i$. Then, equation \eqref{pseudo_total_law} yields that:
		\begin{align}
			\pp{i} - p_i &= p_ip_{\setminus i} + (1-p_i)(1-p_{\setminus i}) - p_i \nn \\  
			&= (1-2p_i)(1-p_{\setminus i}) < 0, \label{p1} \\
			\pp{i} -\nicefrac{1}{2}&= (p_i-\nicefrac{1}{2})p_{\setminus i} + (\nicefrac{1}{2}-p_i)(1-p_{\setminus i}) \nn \\ 
			&= (p_i-\nicefrac{1}{2})(2p_{\setminus i}-1)>0, \label{p2}.
		\end{align}
		Hence, if $p_i>\nicefrac{1}{2}$, $\forall i$, then $\nicefrac{1}{2}<\pp{i}<p_i$.
	\end{proof}
	%%%%%%%%%%%%%%%%%%%%%%%%%%%%%%%%%%%%%%%%%%%%%%%%%%%%%%%%%%%%%%%%%%%%%%%%%%%%%%%%%%%%
	\section{Proof of Proposition \ref{pseudopropmixed}}
	\label{app:mixedpseudo}
	\begin{proof}[Proof of Part (1)]
		Equation \eqref{pseudocompetencediff} indicates that any committee satisfying
		\begin{equation*}
			\condprob{\V{\Xnot{j}}=Y}{X_i\neq Y} > \condprob{\V{\Xnot{j}}\neq Y}{X_i=Y}
		\end{equation*}
		preserves ordering. For every consistent committee, $\exists N^{*}$ such that $\forall N>N^{*}$
		\begin{equation*}
			\prob{\sum_{k\neq i,j} \eta_{k} > \floor*{\frac{N-2}{2}}} > \prob{\sum_{k\neq i,j} \eta_{k} < \ceil*{\frac{N-2}{2}}},
		\end{equation*}
		which yields that pseudo competence preserves ordering for consistent mixed committees.
	\end{proof}
	%%%%%%%%%%%%%%%%
	%%%%%%%%%%%%%%%%
	%%%%%%%%%%%%%%%%
	\begin{proof}[Proof of Part (2)]
		For every consistent committee, $\exists N^{*}$ such that $\forall N> N^{*}$, $\pn{i}>\nicefrac{1}{2}$. Recall equations \eqref{p1}-\eqref{p2}, which yield that:
		\begin{align*}
			\nicefrac{1}{2} &< \pp{i} < p_i \hspace{1em} \text{if} \hspace{1em} p_i>\nicefrac{1}{2}, \\
			\nicefrac{1}{2} &> \pp{i} > p_{i} \hspace{1em} \text{if} \hspace{1em} p_i<\nicefrac{1}{2}, \\
			\nicefrac{1}{2} &= p_i = \pp{i} \hspace{1em} \text{if} \hspace{1em} p_i=\nicefrac{1}{2}.
		\end{align*}
		Therefore, $\min\lb \pp{i},1-\pp{i}\rb \geq \min\lb p_i,1-p_i\rb$ for consistent mixed committees.
	\end{proof}
	%%%%%%%%%%%%%%%%%%%%%%%%%%%%%%%%%%%%%%%%%%%%%%%%%%%%%%%%%%%%%%%%%%%%%%%%%%%%%%%%%%%
	\section{Proof Theorem \ref{improved_upper}}
	\label{app:thm_improved_NB}
	\begin{proof}[Proof Theorem \ref{improved_upper}]
		Let $w_i = \log\nicefrac{p_i}{q_i}$. Chernoff bounding technique yields that  \cite[Chapter 2.2.1]{raginsky2013concentration}:
		\begin{equation*}
			\prob{f^{NB}\yay{\rvec{X}} \neq Y} \leq e^{-t\Phi} \expt{\exp\yay{-t \sum_{i=1}^{N}w_i \yay{\eta_i-p_i}}}.
		\end{equation*}
		Observe that the expectation is with respect to $\eta_i\sim\mathcal{B}\yay{p_i}$:
		\begin{equation*}
			\expt{e^{-t \sum\limits_{i}w_i \yay{\eta_i-p_i}}} = p_i e^{-q_i w_i t} + q_i e^{p_i w_i t}.
		\end{equation*}
		Therefore, the probability of error for the NB rule:
		\begin{align}
			\prob{f^{NB}\yay{\rvec{X}} \neq Y} 
			&\leq e^{-t\Phi} \prod_{i} \yay{p_i e^{-q_i w_i t} + q_i e^{p_i w_i t}}  \nn \\ 
			\label{target_local}
			&\hspace{-5em}= \exp\underbrace{\yay{\textstyle -t\Phi + \sum_{i} \log\yay{p_i e^{-q_i w_i t} + q_i e^{p_i w_i t}}}}_{\triangleq -t\Phi+\phi(t;\dvec{p})}.  
		\end{align}
		The derivative of $-t\Phi+\phi(t;\dvec{p})$ is given in \eqref{big_derivative}, which yields that \eqref{target_local} is minimized when $t=1$. Hence,
		\begin{figure*}[t]
			\begin{equation}
				\label{big_derivative}
				\frac{\partial}{\partial t} \yay{\phi(t;\dvec{p})-t\Phi}
				= \sum_i w_i\brac{\frac{-p_i q_i e^{-q_i w_i t} + p_i q_i e^{p_i w_i t}}{p_i e^{-q_i w_i t} + q_i e^{p_i w_i t}} - \yay{p_i -\nicefrac{1}{2}}}
				= \sum_i \frac{w_i}{2}\brac{\frac{-p_i e^{-q_i w_i t} + q_i e^{p_i w_i t}}{p_i e^{-q_i w_i t} + q_i e^{p_i w_i t}}}		
			\end{equation}
		\end{figure*}
		\begin{align*}
			\prob{f^{NB}\yay{\rvec{X}} \neq Y} 
			&\leq \exp\yay{-\Phi + \phi(1;\dvec{p})} \\
			&= \exp\yay{-\Phi + \sum_i \log 2q_i \yay{\frac{p_i}{q_i}}^{p_i}} \\ 
			&= \exp\yay{\sum_i \log 2\sqrt{q_i p_i}},
		\end{align*}
		yielding the asserted bound.
	\end{proof}
	It is important to note that this bound is the sharpest possible using Chernoff bounding technique and it is a direct consequence of the weight function $w(p) = \log\nicefrac{p}{q}$. 
	%%%%%%%%%%%%%%%%%%%%%%%%%%%%%%%%%%%%%%%%%%%%%%%%%%%%%%%%%%%%%%%%%%%%%%%%%%%%%%%%%%%%%%%%%%%%%%
	\section{Proof of Theorem \ref{limitperformance}}
	\label{app:PNBproof}
	\begin{proof}[Proof of Theorem \ref{limitperformance}]
		Let $\varepsilon_i = p_i-\nicefrac{1}{2}$, hence $\pp{i} = \nicefrac{1}{2}+2\varepsilon_i\epsn{i}$ and observe that
		\begin{equation*}
			\frac{\pp{i}}{1-\pp{i}} = \frac{4\varepsilon_i\epsn{i}+1}{1-4\varepsilon_i\epsn{i}} = \frac{1+2\varepsilon_{i}+\left(\nicefrac{1}{2\epsn{i}}-1\right)}{1-2\varepsilon_{i}+\left(\nicefrac{1}{2\epsn{i}}-1\right)}.
		\end{equation*} 
		Then, the ratio of the weights of PNB and NB decision rules are as follows: 
		\begin{align*}
			1
			&\geq \frac{\log\frac{4\varepsilon_i\epsn{i}+1}{1-4\varepsilon_i\epsn{i}}}{\log\frac{2\varepsilon_{i}+1}{1-2\varepsilon_{i}}} = \frac{\log\frac{1+2\varepsilon_{i}+\left(\nicefrac{1}{2\epsn{i}}-1\right)}{1-2\varepsilon_{i}+\left(\nicefrac{1}{2\epsn{i}}-1\right)}}{\log\frac{1+2\varepsilon_{i}}{1-2\varepsilon_{i}}} \\
			&\geq 1- \left(\nicefrac{1}{2\epsn{i}}-1\right)\underbrace{\frac{4\varepsilon_i}{\left(1-4\varepsilon_i^2\right)\log\frac{1+2\varepsilon_{i}}{1-2\varepsilon_{i}}}}_{C(\abs{\varepsilon_i})}
		\end{align*}
		The inequality follows from the Taylor series expansion of $\log\frac{1+a+x}{1-a+x}$ with respect to variable $x$ around $x=0$, which corresponds to $\epsn{i}\approx \nicefrac{1}{2}$. Observing that $\epsn{i}\geq a(N)-\nicefrac{1}{2}$, $\forall N, i$ and $C(\varepsilon_i)\equiv C(\abs{\varepsilon_i})$ is monotonic in $\abs{\varepsilon_i}$ yield that:
		\begin{equation}
			\label{importantratio}
			\frac{\tilde{\Phi}}{\Phi}= \frac{\sum_{i=1}^{N}\left(\varepsilon_{i}\log\frac{4\varepsilon_i\epsn{i}+1}{1-4\varepsilon_i\epsn{i}}\right)}{\sum_{i=1}^{N}\varepsilon_{i}\log\frac{2\varepsilon_{i}+1}{1-2\varepsilon_{i}}} \geq 1- C(\nicefrac{1}{2}-\gamma) \frac{1-a(n)}{a(n)-\nicefrac{1}{2}}.
		\end{equation}
	\end{proof}
	%%%%%%%%%%%%%%%%%%%%%%%%%%%%%%%%%%%%%%%%%%%%%%%%%%%%%%%%%%%%%%%%%%%%%%%%
	\subsection{Proof of Corollary \ref{fixedsizepotential}}
	\begin{proof}[Proof of Corollary \ref{fixedsizepotential}]
		Similar to the proof of Theorem \ref{limitperformance}, consider the ratio in \eqref{importantratio} and observe that:
		\begin{equation*}
			C(\nicefrac{1}{2}-\gamma) \left(\nicefrac{1}{2\epsn{i}}-1\right) \leq \delta, \forall i,
		\end{equation*}
		ensures that $\frac{\tilde{\Phi}}{\Phi}\geq 1-\delta$. Change of variables $\epsn{i} = \pn{i}-\nicefrac{1}{2}$ concludes proof.
	\end{proof}
	%%%%%%%%%%%%%%%%%%%%%%%%%%%%%%%%%%%%%%%%%%%%%%%%%%%%%%%%%%%%%%%%%%%%%%%%
	\subsection{Proof of Equation \eqref{performancePNB}}
	\label{PNBperf}
	The following proof is based on \cite[Theorem 1]{berend2015finite}. We go over the algebraic manipulations necessary in order to prove \eqref{performancePNB}. We first show that:
	\begin{equation}
		\label{upp}
		\prob{f^{PNB}(\mathbf{X})\neq Y} \leq \exp\left(-\frac{\tilde{\Phi}}{2}\right)
	\end{equation}
	Observe that allowing $w(\pp{i})\triangleq \frac{\pp{i}}{1-\pp{i}}$ and $\xi \triangleq \ind{X_i = Y}\sim \mathcal{B}(p_i)$:
	\begin{align*}
		&\prob{f^{PNB}(\mathbf{X})\neq Y}= \\ 
		& \prob{\sum \xi_i w(\pp{i}) - \expt{\sum \xi_i w(\pp{i})}\leq -\sum \left(\frac{1}{2}-p_i\right)w(\pp{i})}
	\end{align*}
	Subsequent application of Kearns-Saul inequality yields \eqref{upp}. The use of Kearns-Saul inequality yields sufficiently sharp bounds for the performance of NB (and PNB) decision rules and it is discussed in detail \cite{berend2015finite}. Next, a lower bound is needed:
	\begin{equation*}
		\prob{f^{PNB}(\mathbf{X})\neq Y} \geq \frac{\nicefrac{3}{4}}{1+\exp\left(2\tilde{\Phi}+4\sqrt{\tilde{\Phi}}\right)}.
	\end{equation*}
	Let $\eta_i \triangleq 2\ind{X_i=Y}-1$ and observe that:
	\begin{align*}
		\expt{\sum \eta_iw(\pp{i})} &= \sum(p_i-q_i)w(\pp{i}) = 2\tilde{\Phi} \\
		\var\left(\sum \eta_iw(\pp{i}) \right) &=  \sum p_iq_iw(\pp{i})^2 \leq 4\tilde{\Phi}
	\end{align*}
	The upper bound on the variance is not straightforward. It follows from $w(\pp{i})\leq w(p_i)$ (which holds $\forall i$ such that $\pn{i}>\nicefrac{1}{2}$ and from consistency $\forall N>N^{*}$ for some $N^{*}$ it holds $\forall i$) and \cite[Lemma 4]{berend2015finite}. The rest follows from observing that:
	\begin{equation*}
		\exp(\sum \eta_iw(\pp{i})) = \prod_{i: \eta_i=1} \frac{\pp{i}}{1-\pp{i}}\prod_{i: \eta_i=-1} \frac{1-\pp{i}}{\pp{i}},
	\end{equation*}
	and repeating the exact same steps as the proof of \cite[Theorem 1(ii)]{berend2015finite}, which we will not repeat here (the fact that $\min\lb \pp{i},1-\pp{i}\rb \geq \min\lb p_i,1-p_i\rb$ is useful). 
	%%%%%%%%%%%%%%%%%%%%%%%%%%%%%%%%%%%%%%%%%%%%%%%%%%%%%%%%%%%%%%%%%%%%%%%%%%%
	\section{Proof of Proposition \ref{lowsamplingconsistency}}
	\label{app:prop_consistency_low}
	\begin{proof}[Proof of Proposition \ref{lowsamplingconsistency}]
		The proof follows from the Hoeffding's inequality, since for any weighted mixture:
		\begin{equation*}
			\prob{\sum_{i=1}^{N}w_i \eta_i < 0} \leq \exp\left(-\frac{8}{N}\left(\sum_{i=1}^{N}\varepsilon_{i}w_{i}\right)^2\right)
		\end{equation*}
		The definition of the pseudo-competence \eqref{pseudocomp} and $w_i = \yay{\pp{i}-\nicefrac{1}{2}}$ yield that the following is sufficient for the consistency of the rule \eqref{lightload}:
		\begin{equation*}
			\frac{1}{\sqrt{N}}\sum_{i=1}^{N}2\varepsilon_{i}^2\epsn{i} \rightarrow \infty \Rightarrow \prob{f^{L}(\Xvec{})\neq Y} \rightarrow 0.
		\end{equation*}
		Similarly by allowing $w_i =1$, $\forall i$, the following is sufficient for a committee to be not asymptotically weak under majority vote:
		\begin{equation*}
			\frac{1}{\sqrt{N}}\sum_{i=1}^{N}\varepsilon_{i} \geq \sqrt{\frac{\log 2}{8}} \Rightarrow \exp\left(-\frac{8}{N}\left(\sum_{i=1}^{N}\varepsilon_{i}\right)^2\right)\leq \frac{1}{2},
		\end{equation*} 
		which ensures that $\exists N^{*}$ such that $\forall N> N^{*}$, $\epsn{i}\geq \delta > 0$, $\forall i$, yielding that:
		\begin{equation*}
			\sum_{i=1}^{N}2\varepsilon_{i}^2\epsn{i}\geq \delta\sum_{i=1}^{N}2\varepsilon_{i}^2 \rightarrow \infty, 
		\end{equation*} 
		as long as $\lim_{N\rightarrow\infty}  \frac{1}{\sqrt{N}}\sum_{i=1}^{N}(p_{i}-\nicefrac{1}{2})^2 = \infty$, which concludes that empirical PNB decision rule is consistent.
	\end{proof} 
	%%%%%%%%%%%%%%%%%%%%%%%%%%%%%%%%%%%%%%%%%%%%%%%%%%%%%%%%%%%%%%%%%%%%%%%%%%%%%
	\section{Proof of Theorem \ref{highsampling}}
	\begin{proof}[Proof of Lemma \ref{deviation}]
		Consider $\left|w_i-\tilde{w}_i\right|$ for an absolutely balanced committee; the following holds $\forall p_i\in\left(\gamma,1-\gamma\right)$:
		\begin{align*}
			&\abs{\log\frac{1+2\varepsilon_{i}}{1-2\varepsilon_{i}} - \log\frac{4\varepsilon_i\epsn{i}+1}{1-4\varepsilon_i\epsn{i}}} \\
			&\leq \left(\nicefrac{1}{2\epsn{i}}-1\right) \frac{4\varepsilon_i}{\left(1-4\varepsilon_i^2\right)} \leq \left(\nicefrac{1}{2\epsn{i}}-1\right) \frac{\nicefrac{1}{2}-\gamma}{\gamma(1-\gamma)}
		\end{align*}
		As long as the right hand side is upper bounded by $\frac{\epsilon}{2}$, $\norm{\mathbf{w}-\tilde{\mathbf{w}}}_1<\frac{\epsilon N}{2}$.
	\end{proof}
	
	\begin{proof}[Proof of Theorem \ref{highsampling}]
		This proof is an extension of \cite[Theorem 11]{berend2015finite}. Consider the following:
		\begin{align*}
			\left|\mathbf{w}\cdot\bm{\eta}-\mathbf{\tilde{w}}(T)\cdot\bm{\eta}\right| 
			&= \left|\mathbf{w}\cdot\bm{\eta}-\mathbf{\tilde{w}}\cdot\bm{\eta}+\mathbf{\tilde{w}}\cdot\bm{\eta}-\mathbf{\tilde{w}}(T)\cdot\bm{\eta}\right| \\ 
			&\hspace{-1em}\leq\left|\mathbf{w}\cdot\bm{\eta}-\mathbf{\tilde{w}}\cdot\bm{\eta}\right|+\left|\mathbf{\tilde{w}}\cdot\bm{\eta}-\mathbf{\tilde{w}}(T)\cdot\bm{\eta}\right|\\
			&\hspace{-1em}\leq\norm{\mathbf{w}\cdot\bm{\eta}-\mathbf{\tilde{w}}\cdot\bm{\eta}}_1 + \norm{\mathbf{\tilde{w}}\cdot\bm{\eta}-\mathbf{\tilde{w}}(T)\cdot\bm{\eta}}_1
		\end{align*}
		The first inequality follows from the triangle inequality and the second inequality follows from the H\"older's inequality, then, \cite[eqn. (41)]{berend2015finite} yields that:
		\begin{align*}
			&\prob{\mathbf{\tilde{w}}(T)\cdot\bm{\eta}\leq 0} \leq \prob{\mathbf{w}\cdot\bm{\eta}\leq \epsilon N} \\
			&+   \prob{\norm{\mathbf{w}\cdot\bm{\eta}-\mathbf{\tilde{w}}\cdot\bm{\eta}}_1 + \norm{\mathbf{\tilde{w}}\cdot\bm{\eta}-\mathbf{\tilde{w}}(T)\cdot\bm{\eta}}_1>\epsilon N}
		\end{align*}
		As long as a committee satisfies the condition in Lemma \ref{deviation}, this upper-bound boils down to:
		\begin{align*}
			\prob{\mathbf{\tilde{w}}(T)\cdot\bm{\eta}\leq 0} &\leq \prob{\mathbf{w}\cdot\bm{\eta}\leq \epsilon N} \\ &+ \prob{\norm{\mathbf{\tilde{w}}\cdot\bm{\eta}-\mathbf{\tilde{w}}(T)\cdot\bm{\eta}}_1>\nicefrac{\epsilon N}{2}}
		\end{align*}
		Now, \cite[Corollary 10]{berend2015finite} yields that $\forall \delta \in (0,1)$ and $\forall i\in[N]$ as long as  
		\begin{equation}
			\label{condition}
			T\min\lb \pp{i}, (1-\pp{i})\rb \geq 3\left(\frac{4}{\sqrt{4\epsilon+1}-1}\right)^2 \log\frac{8N}{\delta},
		\end{equation}
		The probability that empirical pseudo weights deviate from pseudo weights are bounded:
		\begin{equation*}
			\prob{ \norm{\mathbf{\tilde{w}}\cdot\bm{\eta}-\mathbf{\tilde{w}}(T)\cdot\bm{\eta}}_1>\frac{\epsilon N}{2}} < \delta
		\end{equation*} 
		Finally, by Property \ref{pseudopropmixed}, $\min\lb p_i, (1-p_i)\rb$ satisfying \eqref{condition} yields that:
		\begin{equation*}
			\prob{f^{H}(\mathbf{X})\neq Y} \leq \delta + \exp\left[-\frac{\left(2\Phi -\epsilon N\right)^2}{8\Phi}\right].
		\end{equation*}
		Observe that Lemma \ref{deviation} and eqn. \eqref{condition} are connected to consistency and absolute balance conditions respectively. Therefore, consider a consistent committee with rate $a(N)$ and observe that Lemma \ref{deviation} holds as long as:
		\begin{equation}
			\label{datdat}
			\frac{\epsilon R(\gamma)}{2} > \rho(N).
		\end{equation}
		Observing that \eqref{condition} is merely absolute balance condition with $\gamma =  \frac{3}{T}\log\frac{8N}{\delta}\left(\frac{4}{\sqrt{4\epsilon+1}-1}\right)^2$, plugging into \eqref{datdat}, and taking Taylor series expansion yields (a long algebraic manipulation that we skip here) yields that Lemma \ref{deviation} holds as long as:
		\begin{equation*}
			\epsilon > \left(\rho(N)\frac{12}{T}\log\frac{8N}{\delta}\right)^{\nicefrac{1}{3}}.
		\end{equation*}
		Defining $C(\delta;N,T) = \frac{12}{T}\log\frac{8N}{\delta}$ concludes the proof.
	\end{proof}
	%%%%%%%%%%%%%%%%%%%%%%%%%%%%%%%%%%%%%%%%%%%%%%%%%%%%%%%%%%%%
	\section{Proof of Theorem \ref{online}}
	\begin{proof}[Proof of Theorem \ref{online}]
		Due inter-worker and inter-task independence, the empirical pseudo na\"{i}ve Bayes decision rule at any given time, $\tau\in[T]$ evolves in a well-defined filtration. Hence, $f^{H}_{\tau}$, the empirical decision rule using weights $\tilde{w}(\tau)$, obeys $\forall t\in (\tau+1,...,T)$:
		\begin{align*}
			&\prob{R(\tau)\cap \lb f^{H}(\mathbf{X}(t>\tau)\neq Y)\rb} \\ &\hspace{2cm}= \mathbb{P}_{\mathbf{X}_{1}^{\tau},\bm{\eta}}\left(R(\tau)\cap \lb \tilde{\mathbf{w}}(\tau)\cdot\bm{\eta}\leq 0\rb\right)\\ 
			&\hspace{2cm}= \mathbb{E}_{\mathbf{X}_{1}^{\tau}}\left[\ind{R(\tau)} \mathbb{P}_{\bm{\eta}}\left( \tilde{\mathbf{w}}(\tau)\cdot\bm{\eta}\leq 0\right)\right]  
		\end{align*}
		It is important to observe that as long as the tasks are static, these probabilities are a function of $\tau$ and the committee profile $\mathbf{p}$. Let $\tilde{\bm{\eta}}_{\tau}$ be a random vector with elements distributed independently with Bernoulli $\pp{i}(\tau)$ and denote $\Delta(\tau) \triangleq \sum_{i=1}^{N}\left|p_i-\pp{i}(\tau)\right|$ In other words, it is a random vector with a pseudo committee profile. A standard tensorization result from \cite{kontorovich2012obtaining,berend2015finite} yields:
		\begin{equation*}
			\left|\mathbb{P}_{\bm{\eta}}\left( \tilde{\mathbf{w}}(\tau) \cdot\bm{\eta}\leq 0\right)-\mathbb{P}_{\tilde{\bm{\eta}}_{\tau}}\left( \tilde{\mathbf{w}}(\tau) \cdot\tilde{\bm{\eta}}_{\tau}\leq 0\right)\right|\leq \Delta(\tau)
		\end{equation*}
		Then $\forall \tau\in[T]$, $\mathbb{P}_{\tilde{\bm{\eta}}_{\tau}}\left( \tilde{\mathbf{w}}(\tau)\cdot\tilde{\bm{\eta}}_{\tau}\leq 0\right)$ is the probability of error for the na\"{i}ve Bayes decision rule of committee strength $\tilde{\Phi}(\tau) \triangleq \sum_{i=1}^{N}\left(\pp{i}(\tau)-\frac{1}{2}\right)\log\frac{\pp{i}(\tau)}{1-\pp{i}(\tau)}$. Therefore, from \cite{berend2015finite}:
		\begin{equation*}
			\mathbb{P}_{\tilde{\bm{\eta}}_{\tau}}\left( \tilde{\mathbf{w}}(\tau)\cdot\tilde{\bm{\eta}}_{\tau}\leq 0\right) \leq \exp\left(-\frac{1}{2}\tilde{\Phi}\left(\tau\right)\right)
		\end{equation*}
		Hence, $\mathbb{P}_{\bm{\eta}}\left( \tilde{\mathbf{w}}^{HS} (\tau) \cdot\bm{\eta}\leq 0\right)\leq\Delta(\tau)+ \exp\left(-\frac{1}{2}\tilde{\Phi}\left(\tau\right)\right)$. Then, by the triangle inequality with a mean absolute deviation estimate from \cite{berend2013sharp,berend2015finite}, we see that:
		\begin{equation*}
			\mathbb{E}_{\mathbf{X}_{1}^{\tau}}\left[\Delta(\tau)\right] \leq \sum_{i=1}^{N}\left|p_i-\pp{i}\right| + \frac{N}{\sqrt{T}}
		\end{equation*}
		This concludes the proof.
	\end{proof}
	\bibliographystyle{IEEE_ECE}
	\bibliography{expertsbib}
\end{document}